\documentclass[pdflatex,sn-nature]{sn-jnl}

\usepackage{graphicx}%
\usepackage{multirow}%
\usepackage{amsmath,amssymb,amsfonts}%
\usepackage{amsthm}%
\usepackage[scr=rsfso]{mathalpha}
\usepackage[title]{appendix}%
\usepackage{xcolor}%
\usepackage{textcomp}%
\usepackage{manyfoot}%
\usepackage{booktabs}%
\usepackage{algorithm}%
\usepackage{algorithmicx}%
\usepackage{algpseudocode}%
\usepackage{listings}%
\usepackage{etoc}

\theoremstyle{thmstyleone}%
\newtheorem{theorem}{Theorem}
\theoremstyle{thmstyletwo}%

\theoremstyle{thmstylethree}%

\begin{document}

\title[Article Title]{Emergence of Fibrations, Compression, and Symmetry Breaking in Artificial Neural Networks}


\author[1]{\fnm{Osvaldo M.} \sur{Velarde}}\email{ovelarde@ccny.cuny.edu}
\author[2]{\fnm{Lucas C.} \sur{Parra}}\email{parra@ccny.cuny.edu}
\author[1]{\fnm{Alireza} \sur{Hashemi}}\email{ahashemi@ccny.cuny.edu}
\author*[1]{\fnm{Hern\'an A.} \sur{Makse}}\email{hmakse@ccny.cuny.edu}

\affil*[1]{\orgdiv{Levich Institute and Physics Department}, \orgname{City College of New York}, \orgaddress{\street{160 Convent Ave}, \city{NY}, \postcode{10031}, \state{NY}, \country{USA}}}

\affil[2]{\orgdiv{Biomedical Engineering Department}, \orgname{City College of New York}, \orgaddress{\street{160 Convent Ave}, \city{NY}, \postcode{610101}, \state{NY}, \country{USA}}}


\abstract{Artificial neural networks are often regarded as powerful yet opaque ``black boxes". Here, we demonstrate that learning in deep neural networks generates local symmetries known in graph theory as fibrations and coverings. We prove that covering symmetries are stable attractors of stochastic gradient descent. Consistent with this theory, we report the emergence of covering symmetries across major network architectures, including multi-layer, convolutional, recurrent, and transformer networks. Exploiting these symmetries enables drastic model compression — reducing networks to 17\% of their original size without sacrificing performance. Furthermore, controlled breaking of covering symmetry overcomes the loss of plasticity, achieving state-of-the-art performance in continual learning. The theoretical results provide a new foundation for AI systems based on symmetries that convert black boxes into interpretable ``colored graphs" and enable more efficient inference and lifelong learning.}


\maketitle

\section{Main}

Despite the development of increasingly powerful neural network architectures, our understanding of their internal structure remains limited. They are often regarded as black boxes, unable to provide explanations for how their network connections capture the regularity of training data. In the absence of a theoretically grounded understanding of learning, the recent surge in artificial intelligence (AI) has been driven predominantly by empirical scaling laws \cite{kaplan2020scaling,hoffmann2022training}, resulting in ever increasing model sizes.

The reliance on scaling as a substitute for understanding creates significant challenges. Training with single large runs has massive cost and energy demands, and when trained sequentially, networks can suffer from a loss of plasticity over time
\cite{dohare2024loss,lyle2023understanding,lyle2024disentangling}.
Additionally, training large models from scratch is data-inefficient, requiring exceedingly large datasets. Without principled guidelines, architecture design becomes a costly trial-and-error process. Finally, when models are excessively over-parameterized, paradoxically, they still perform well in practice \cite{sejnowski2020unreasonable}, a subject of ongoing theoretical debate \cite{belkin2019reconciling,nakkiran2020deep}.

To address these challenges, we propose a new theoretical formalism that analyzes the internal structure of deep neural networks through the lens of graph symmetries. While global symmetry groups are fundamental to geometric deep learning (GDL) \cite{bronstein2021geometric} and our understanding of theoretical physics \cite{bronstein2021geometric, wilczek2016beautiful,
weinberg1995thequantum}, they are too rigid to capture the diversity observed in both artificial and biological neural systems \cite{makse2026symmetry,gili2025fibration,
  avila2024symmetries, velarde2024role}. For instance, existing machine learning applications rely heavily on strict global symmetries, such as shift equivariance in convolutional neural networks (CNNs) \cite{lecun1998gradient,cohen2016group} or permutation equivariance
\cite{Kipf2016semi-supervised,satorras2022en,makse1992thethermodynamics}
in graph neural networks (GNNs) \cite{bronstein2021geometric}. Instead, we identify less restrictive local symmetries that naturally emerge in the input and output trees of computational graphs. These local symmetries are known in graph theory as fibrations, opfibrations, and coverings \cite{boldi2002fibrations,makse2026symmetry}. Originally introduced by Grothendieck as maps in category theory \cite{grothendieck1959technique}, they were later adapted for graphs \cite{boldi2002fibrations,morone2020fibration}.

Here, we provide a unified mathematical framework demonstrating that stochastic gradient descent (SGD) acts as a mechanism for local symmetry formation. We prove mathematically that covering symmetries emerge in SGD due to the "synchronized learning" of weight updates. This synchronized learning induces fibration symmetries, providing a formal explanation for the emergence of neural activity synchronization, which has often been reported in deep networks \cite{velarde2024role,papyan2020prevalence,doimo2022redundant}. Furthermore, we develop a "balance coloring" algorithm \cite{kamei2013computation} to identify these structures in the weights of deep networks, and use the resulting fibrations for substantial compression of trained models with minimal impact on performance. We validate the emergence of these local symmetries across a wide range of architectures, including multilayer perceptrons (MLP), CNNs, recurrent networks (e.g., long short-term memory, LSTM), and Transformers in both supervised and reinforcement learning (RL) settings. Finally, by carefully breaking these emergent symmetries, we expand the capacity of the network to overcome loss of plasticity \cite{lyle2023understanding,lyle2024disentangling}, demonstrating state-of-the-art performance in continual learning \cite{dohare2024loss}.

\subsection*{An hierarchy of symmetries in computational graphs}
The symmetries we will discuss form a hierarchy and emerge in a
variety of network structures. For an easy visualization, consider a layered feedforward network with binary weights (where the connections
are 1 or 0), as shown in Fig.~\ref{fig:hierarchy}. In the figure,
nodes are colored to indicate their class according to
various symmetries. We will generalize these concepts to
continuous-valued weights in the next section.

\begin{figure}[htbp]
\centering
\includegraphics[scale=0.15]{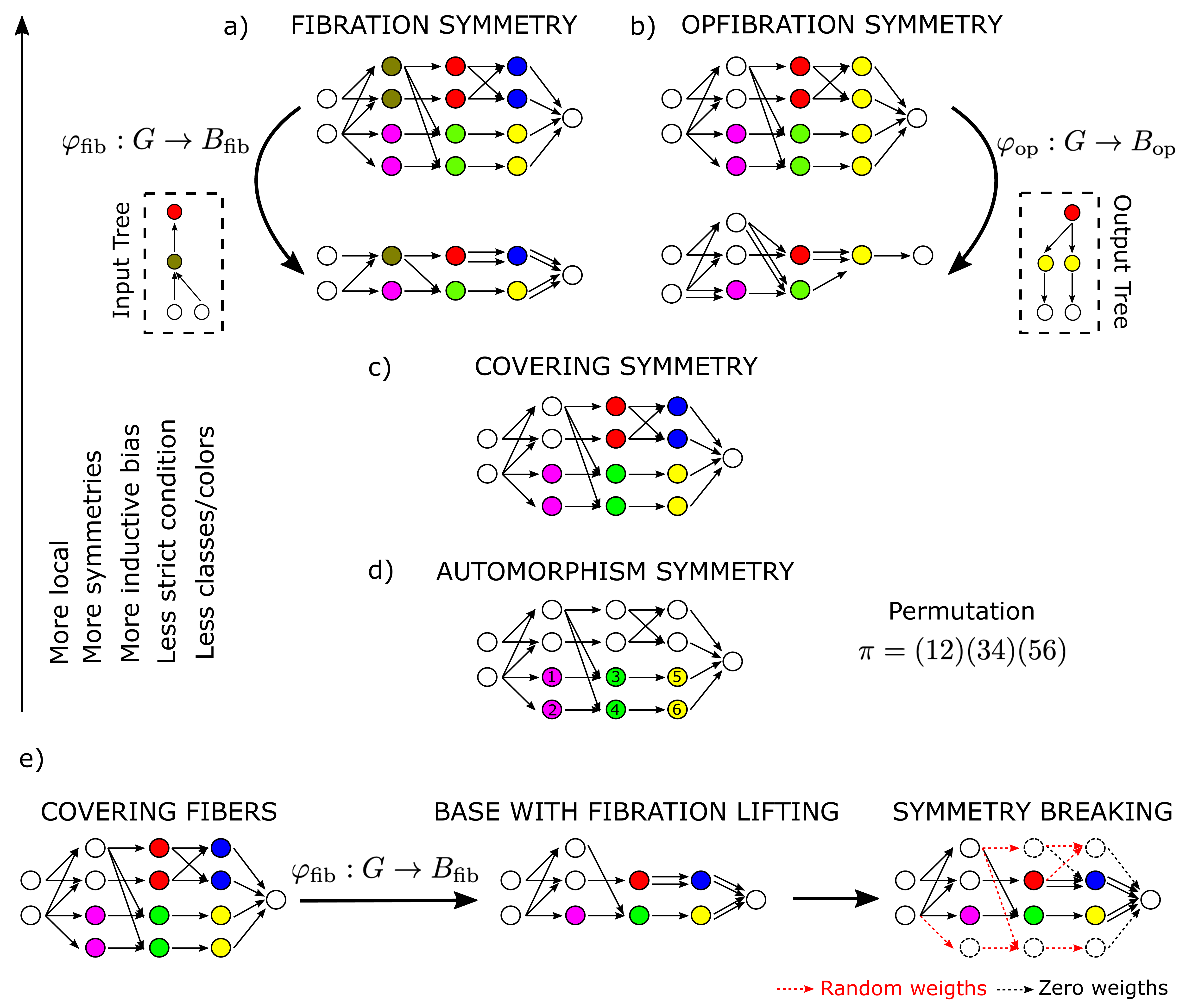}
\caption{\textbf{Hierarchy of symmetries in a feedforward
    network.} For simplicity, weights are binary, and connections with zero
  weight are not plotted. Conventional automorphism permutation is
  exemplified in the case of a graph in panel \textbf{(d)}. It identifies three
  pairs of symmetric nodes by a permutation $\pi$.  The proposed framework
  is based on the generalization to local symmetries, such as
  (op)fibrations and coverings (\textbf{a-c}); see text for detail. These symmetries have less strict conditions that can be satisfied by a larger number of nodes resulting in fewer symmetry classes (colors) (blank nodes should be
  interpreted as distinct colors). The reduced degrees of freedom increases the inductive bias. 
  \textbf{(e)} Breaking of symmetry for
  continual learning consists of two steps. Compress to
  the covering base 
  (middle) to preserve the learned
  task. Randomize or zero out redundant weights (right) to provide new degrees of freedom to continue
  learning. }
\label{fig:hierarchy}
\end{figure}

{\bf Fibration Symmetries:} Nodes possess a fibration symmetry and are said to belong to the
same \textit{fiber} if they have isomorphic input trees
\cite{makse2026symmetry,boldi2002fibrations,morone2020fibration}. This
means that the entire structure of connections from the graph's input
to the nodes is isomorphic (see Methods \ref{material:symmetries}).
Even though the input tree spans all the network, it represents the
local 'view' of the node at its root \cite{boldi2002fibrations}. In
practice, fibers are identified by a balanced coloring algorithm from
graph theory \cite{makse2026symmetry,kamei2013computation}. This
algorithm partitions the network into balanced coloring classes of
nodes (fiber partition) by iteratively assigning the same color to
nodes that receive the same set of input colors
\cite{golubitsky2006nonlinear}, hence the term \textit{balanced
  coloring}.  An example is shown in Fig.~\ref{fig:hierarchy}a, where
nodes in the same fiber are colored the same. The input tree for a red
node is shown on the left.

Fibration symmetries allow the graph $G$ to be compressed into a
smaller \textit{base} graph $B_{\rm fib}$. This compression
$\varphi_{\rm fib}: G\to B_{\rm fib}$ works by merging all nodes that
share the same color (i.e., belong to the same fiber) while conserving
the input trees (Fig.~\ref{fig:hierarchy}a, Methods
\ref{material:lifting-prop}). We will show that the resulting base
graph $B_{\rm fib}$ performs the identical forward computation 
as the original graph $G$. In deep neural
networks (DNNs), it means that we can compress the network without
losing performance.  The inverse of this compression is called
\textit{lifting} \cite{boldi2002fibrations}, an operation that 
restores the network's original dimensionality without altering 
its functional output (explained in detail in Methods \ref{material:lifting}.

{\bf Opfibration Symmetries:} A similar principle applies when learning the parameters of the network. During error backpropagation --- the canonical learning algorithm in modern AI --- the error at the output is propagated backward through the output tree. An opfibration symmetry occurs when nodes share isomorphic output trees --- the structure of connections leading from them to the output layer (Fig.~\ref{fig:hierarchy}b). 

{\bf Covering Symmetries:} When multiple nodes have isomorphic input {\em and} output trees, they form a covering symmetry and belong to a \textit{cover} (Fig.~\ref{fig:hierarchy}c). The corresponding coloring is the intersection of the fibration and opfibration coloring partitions; i.e., a finer partition of the graph. Clearly, a covering is a more stringent symmetry than fibration or opfibration.

The most strict symmetry is the \textit{automorphism}.
This is a global permutation of the network's nodes
that leaves the entire graph's connectivity unchanged.
An example of nodes with automorphism symmetry (i.e. they are in the same \textit{orbit}) is shown in Fig.~\ref{fig:hierarchy}d along with
the permutation of node labels. Although automorphisms are the
cornerstone of GDL \cite{bronstein2021geometric} and theoretical
physics
\cite{wilczek2016beautiful,weinberg1995thequantum,georgi2000lie}, this
symmetry is so restrictive that we have never observed it in our
trained networks.

In summary, these symmetries form a hierarchy of increasing
strictness: from fibrations and opfibrations, to coverings, and
finally to automorphisms. As conditions become more stringent,
there are more distinct color classes and fewer nodes within each
class (meaning fewer symmetries). Less strict local symmetries, such as 
fibrations, are more common and allow greater compression into a
more compact base. This compression results in a model with fewer
effective degrees of freedom, which corresponds to a stronger
inductive bias. For a more formal discussion on symmetries of a graph, see Methods \ref{material:symmetries}.

\section{Results}

\subsection{Theoretical results: the mathematical foundation of emergent symmetry}

Here we establish that deep learning is not merely a process of parameter-tuning, but one of symmetry formation. Learning organizes networks into local symmetries, consisting of fibrations, opfibrations, and together, forming coverings. In the following, we generalize this hierarchy to weighted computational graphs.

\subsubsection*{\bf Computational graphs and synchronization} 
The hierarchical local symmetries explained above emerge in standard
architectures like MLPs, CNNs, RNNs and Transformers (see Methods
\ref{material:hypergraphs}), but they are most easily understood in
the canonical MLP. Here, weight $W^{(\ell)}_{ik}$ connects node $k$ in layer $\ell-1$ to node
$i$ in layer $\ell$. This creates a directed weighted acyclic computational
graph. The activity $h_i^{(\ell)}$ propagates forward from input $x = h^{(0)}$ to output $y$ through a
weighted sum followed by a nonlinear activation function $\sigma$,
\begin{equation}
  h_i^{(\ell)} = \sigma \left( \sum_k W^{(\ell)}_{ik}
  h_k^{(\ell-1)} \right).
  \label{eq:mlp-activity}
\end{equation}

The network generates an error $L = \mathcal{L}(\hat{y},y)$ with respect to the desired target $\hat{y}$. This error propagates backward \cite{rumelhart1986learning} starting from the output with $\delta_i^{(N)} = \partial L / \partial h_i^{(N)}$, coupling the error with the forward activity by the derivative $\sigma'$ of the activation function:
\begin{equation}
    \delta_i^{(\ell)} = \sigma'^{(\ell)}_i \sum_k W^{(\ell+1)}_{ki}  \delta_k^{(\ell+1)}\, .  
   \label{eq:mlp-backprop}
\end{equation}

\subsubsection*{Fibrations synchronize activity, opfibrations synchronize error}

In a weighted graph, two nodes $i$ and $j$ in layer $\ell$ are in a fibration symmetry ($i \underset{\rm  fib}{\sim} j$) when they have isomorphic input trees, which mean that they receive the same summed input from the fibration colors of the previous layer. In other words, they satisfy an equal-sum criterion of their incoming weights (see Eq. \ref{eq:fibration-definition}). Because real-world networks use continuous weights, this equality constraint is verified up to a threshold $\varepsilon_{\rm fib}$
\begin{equation}
   \bigm| \sum_{k \in c} (W^{(\ell)}_{ik} - W^{(\ell)}_{jk}) \bigm| \leq \varepsilon_{\rm fib} \, ,
  \label{eq:quasi-fiber}
\end{equation}
where the sum is over all fiber colors $c$ of the previous layer.
In the limit $\varepsilon_{\rm fib} = 0$, we obtain an exact fibration symmetry ($i \underset{\rm fib}{\sim} j$). For non-zero $\varepsilon_{\rm fib}$, we refer to these as {\em quasi-fibers} following \cite{boldi2022quasifibrations}. Our first theoretical observation is that two nodes in the same fiber synchronize their activity ($h_i=h_j$) for {\em any} input  $x$ to the network (Fig.~\ref{fig:theoretical_results}a). This theorem is formalized in Methods ~\ref{sec:fibration-symmetry-activity-synchronization} and proven in Methods  \ref{material:proof-synchronization}. For quasi-fibers, the accuracy of the synchronization is bounded by a factor proportional to $\varepsilon_{\rm fib}$ (Methods \ref{material:quasi-synchronization}). Therefore, structural symmetry induces \emph{activity synchronization}. 

In the backward pass, an opfibration symmetry ($i \underset{\rm
  op}{\sim} j$) occurs when nodes share isomorphic output
trees. Because error backpropagation involves a nonlinear coupling
between the error signal and the forward activation slope through
$\sigma'^{(\ell)}_i$ in Eq. (\ref{eq:mlp-backprop}), we introduce
{\em flavored connections} to represent these dependencies
(colored arrows in Fig.~\ref{fig:theoretical_results}b; see Methods~\ref{material:symmetry-ffn} and \ref{sec:opfibration-symmetry-error-synchronization}). 
Our second theoretical result is that flavored
opfibration symmetry ensures \emph{error synchronization}: the
backpropagated error signals in an opfiber (green right sides in
Fig.~\ref{fig:theoretical_results}a) become identical
($\delta_i=\delta_j$) for any input $x$ and output $y$ (theorem proven
in Methods \ref{material:proof-synchronization}). In practice, we will use quasi-opfibers with threshold $\varepsilon_{\rm op}$ similar to the definition presented for quasi-fibers.

\begin{figure}[htbp]
\centering
\includegraphics[scale=0.15]{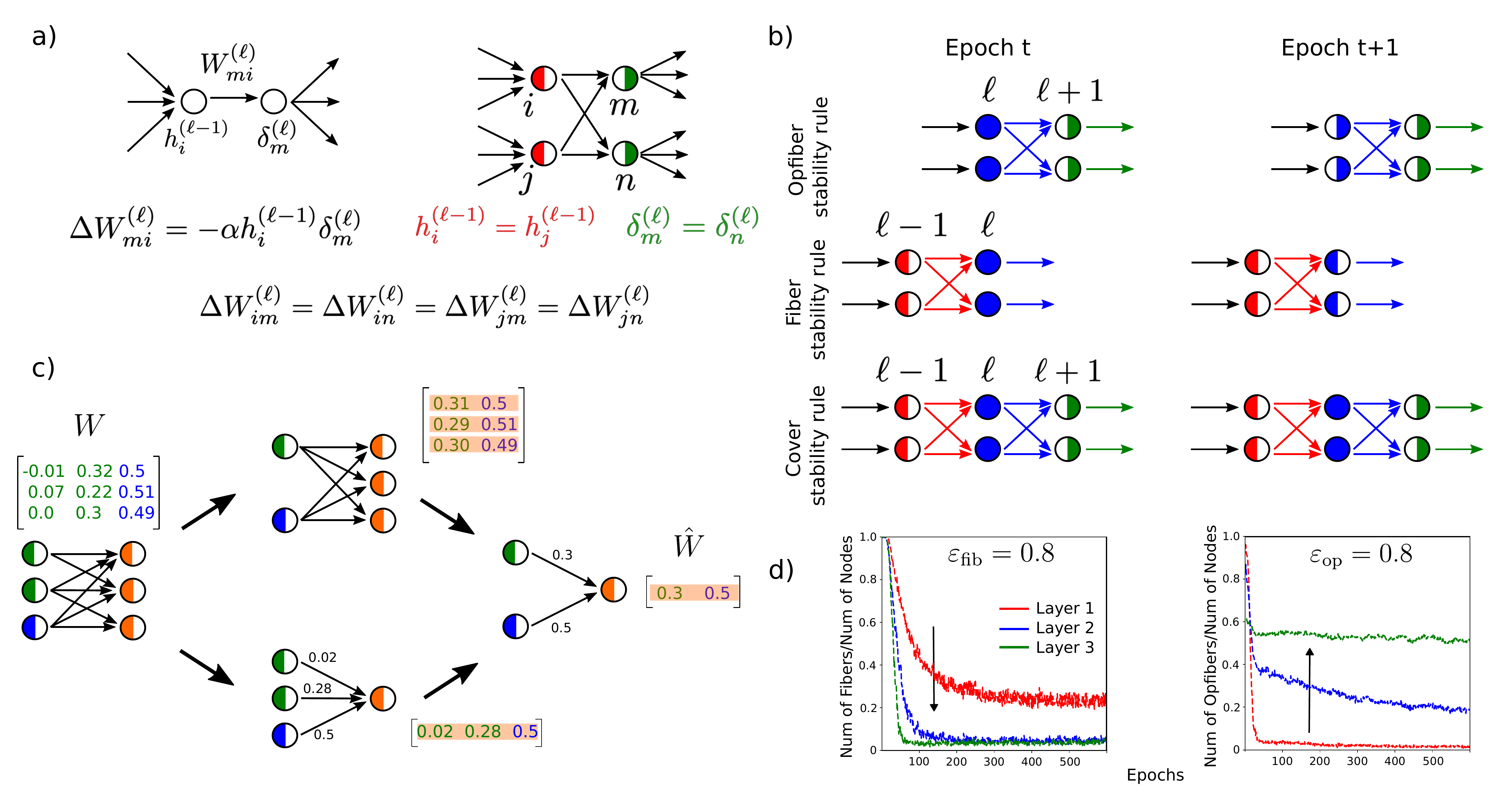}
\caption{\textbf{Summary of theoretical results}. \textbf{(a)} The gradient descent learning rule updates a connection's weight $W^{(\ell)}_{mi}$ based on the product of the activity $h^{(\ell-1)}_i$ and the error signal $\delta^{(\ell)}_m$. Nodes in the same fiber (red left sides) synchronize their activities - Eq. (\ref{eq:synch-activity}). Nodes in the same opfibers (green right sides) synchronize their error signals - Eq.~(\ref{eq:synch-error}). Weights connecting nodes in a fiber with nodes in an opfiber will have the same updates $\Delta W^{(\ell)}_{mi}$ - Eq. (\ref{eq:synch-learning}). \textbf{(b)} Stability rules used to prove the Cover Coarse-Graining Theorem. Connection colors indicate ``flavor" for the opfibers. Black connections for fibrations are not flavored. This introduces an asymmetry, with connection colors defined at the output but not at the input of nodes. \textbf{(c)} Example of fibration compression according to \ref{eq:fib-weight-compression}. Connections with zero weight are not plotted. $\hat{W}$ is calculated based on $W$ using Eq. (\ref{eq:fib-weight-compression}). \textbf{(d)} Number of fibers (left) and opfibers (right) in distinct layers across learning epochs. As nodes group into fibers and opfibers, the number of distinct classes decreases, and they grow in size. Later layers in the networks (black arrow) cluster into fibers more quickly. For opfibers, initial layers cluster more quickly.} 
\label{fig:theoretical_results}
\end{figure}

\subsubsection*{\bf Synchronized learning}

During learning, the gradient of the loss $L$ with respect to the weights $W^{(\ell)}_{ik}$ 
depends locally on the activity of the node and the error signals
\cite{rumelhart1986learning}:

\begin{equation}
    \Delta W^{(\ell)}_{mi} = - \alpha \frac{\partial L}{\partial W^{(\ell)}_{mi}}= -\alpha \delta^{(\ell)}_m \cdot h^{(\ell-1)}_i \,
    ,
    \label{weight-update}
\end{equation}
where $\alpha$ is a learning rate (see Methods \ref{material:gd}). When the input nodes in a layer are in a fiber and the output nodes are in a flavored opfiber, then activity and error are synchronized. Consequently, the corresponding weight updates under gradient descent (GD) become identical (Fig. ~\ref{fig:theoretical_results}a right):

\begin{equation}
    \Delta W^{(\ell)}_{mi} = \Delta W^{(\ell)}_{mj}
    = 
    \Delta W^{(\ell)}_{ni}  = \Delta W^{(\ell)}_{nj}\, .
\label{eq:synch-learning} 
\end{equation}
We refer to this as {\em synchronized learning}. This theorem is proven in  Methods~\ref{material:proof-synchron}).

\subsubsection*{Covers and the emergence of symmetry}

We now turn to the central finding of this work, which is that these symmetries emerge naturally during learning with stochastic gradient decent (SGD).
When a neural network is initialized, its weights are set to
random values. Large random graphs of this kind rarely have
significant global automorphism symmetries as they are forbidden by
Erd\H{o}s-R\'enyi asymmetry theorem \cite{erdos1963asymmetric}.
However, the training process itself appears to create structure. For
example, studies have observed that SGD causes different nodes to
develop similar input and output weights \cite{chen2023stochastic} and
redundancies \cite{papyan2020prevalence,doimo2022redundant}. We have
also previously observed the emergence of \textit{synchronized} nodes during training
\cite{velarde2024role}, suggesting fibration symmetries.

To understand the emergence of these symmetries, we have to consider \emph{covering symmetry} ($i \underset{\rm cov}{\sim} j$), which occurs when the nodes are in both fibration and opfibration symmetry. As with fibers and opfibers, covers partition the nodes in a layer into groups with the same colors, e.g. Fig.~\ref{fig:hierarchy}c. The main theoretical result is that these covering partitions $\mathcal{C}^{\rm cov}$ are preserved under GD, which is formalized in the following Theorem:

\begin{theorem} 
\label{theorem:covering}
\textbf{Cover Coarse-Graining Theorem.} 
\textit{
Under the dynamic of GD, the covering partition $\mathcal{C}^{\rm cov}_{\ell}(t+1)$ is a coarsening of  $\mathcal{C}^{\rm cov}_{\ell}(t)$.
}
\end{theorem}

This theorem (proven in Methods \ref{material:covering-theorem-proof}) dictates that distinct covers can merge to form larger (coarser) covers, but once multiple nodes join a cover, they cannot exit. Therefore, the weights of covers constitute an invariant set. Any parameter set that is invariant under GD becomes a stable attractor under stochastic GD for sufficiently large learning rates  \cite{chen2023stochastic}. We therefore conclude that covers are stable attractors of the SGD dynamics. 

The theoretical results presented so far readily generalize to other computational graphs, including hypergraphs, as we discuss in the Methods \ref{material:hypergraphs}. This includes operations such as convolutions (in CNN), multiplicative gating (in most RNN and LSTM in particular), residual connections, and attention gates, which is the core innovation of transformer networks. 

\subsubsection*{Compression rule preserves activity and loss}

Each symmetry induces a more compact base graph, whereby all nodes of the same color are merged. We propose the following {\em compression rule} for the weights in the base graph, where the connection from color $c$ to color $c'$ is the mean across nodes in $c'$ and the sum across nodes in $c$ (exemplified in Fig.~\ref{fig:theoretical_results}c):
\begin{equation}
    \hat{W}_{c'c}^{(\ell)} = \frac{1}{|c'|} \sum_{i \in c'} \sum_{k \in c} W_{ik}^{(\ell)}.
    \label{eq:fib-weight-compression}
\end{equation}
When colors are obtained using fibration symmetry, we will refer to Eq. \ref{eq:fib-weight-compression} as the \textit{fibration compression rule}; when colors are obtained using covering symmetry, we will refer to it as the \textit{covering compression rule}. 

This compression rule strictly maintains the identical forward computation (Methods \ref{sec:fiber-compression-MLP}), preserving the network's exact loss function. In the case of quasi-fibers and quasi-covers, the activity and loss are preserved only approximately, bounded by the precision and a data-dependent positive coefficient $K^{(\ell)}_{c}$: 
\begin{eqnarray}
\Delta L^{(\ell)} \lesssim K^{(\ell)}_{c} \varepsilon_{\rm op} \varepsilon_{\rm fib}  .
\label{eq:loss-change-bound}
\end{eqnarray}

The compression rule (\ref{eq:fib-weight-compression}) can be applied in different network architectures (Methods \ref{material:hypergraphs}). In CNNs, features are treated as nodes and can be compressed accordingly (Methods \ref{sec:fiber-compression-CNN}, Fig. \ref{fig:compression_architec}a). In RNNs, such as the LSTM, cover symmetries emerge in the gating mechanisms, with the weight matrices of the forget and input gates playing the role of fibers and opfibers (Methods \ref{sec:fiber-compression-hypergraph}, Fig. \ref{fig:compression_architec}b). In transformer networks, symmetries emerge in the weight matrices of the query and key (Methods \ref{sec:fiber-compression-hypergraph}, Fig. \ref{fig:compression_architec}c), while the value matrix can be treated the same as the weight matrices in an MLP.

\subsection{Empirical results: emergence, compression and breaking of symmetries}

In the following sections, we will empirically demonstrate the emergence of cover symmetries during stochastic gradient descent, and argue that this symmetry formation captures regularity in the data.  We also show that these symmetries allow substantial compression of networks with negligible change in performance, while outperforming existing pruning methods. Finally, we argue that the emergence of covering symmetry results in the loss of plasticity in deep networks and show empirically how fiber symmetry breaking can achieve a new state-of-the-art in continual learning.

\subsubsection*{Emergence of symmetries during learning}

First, we provide empirical evidence for the emergence of fibers, opfibers and covers (Fig.~\ref{fig:theoretical_results}d) during SGD for an MLP trained in the MNIST classification problem  (Methods \ref{sec:experiments_cfg}). Since weights are initialized with random values, at the start of learning, all nodes form trivial (single-node) fibers. As learning progresses, the nodes start to partition into multi-node covers, forming fibers, so that the number of unique fibers and opfibers decreases. We observe larger fibers in later layers of the network (Fig.~\ref{fig:theoretical_results}d, left), which is explained in our theory by the increasingly more relaxed conditions for the fibers: Each layer removes an additional degree of freedom whenever two nodes satisfy the fiber equation (\ref{eq:fibration-definition}). As we move across layers, the subspace of solutions increases and the fibers grow in size. Similar reasoning explains why opfibers are larger in the initial layers (Fig.~\ref{fig:theoretical_results}d, right).

The temporal dynamic of cover and fiber formation is illustrated in the accompanying 
\href{https://github.com/makselab/fibrations_in_dnns/blob/main/MNIST_Symmetries/animation.mp4} {video S1}. The video shows that (quasi-) opfibers first appear in the output layers and their emergence propagates backwards across layers, while (quasi-) fibers first emerge in the first layer and propagate to the output. This is consistent with the stability rules (Fig.~\ref{fig:theoretical_results}b), that are the basis of the covering theorem (Methods \ref{material:covering-theorem-proof}). The rules dictate that fibers in one layer preserve fibers in the next, while opfibers preserve opfibers in the preceding layer. 
Thus, opfiber formation tends to propagate backward from the output layer, while fiber formation tends to propagate forward from the input as learning progresses. We confirm this empirically for quasi-fibers and quasi-covers in an MLP in Fig.~\ref{fig:results-formation-and-pruning}a. 

We also demonstrate the emergence of symmetry in a CNN trained in ImageNet classification and in an LSTM trained with reinforcement learning to play the Atari game Beam Rider  (Methods \ref{sec:experiments_cfg}). In the CNN, quasi-covers grow in both convolutional and dense layers as training progresses (i.e. number of unique covers decreases), and classification accuracy improves (Fig. \ref{fig:results-formation-and-pruning}c). For the LSTM, quasi-covers first grow in size, i.e., unique nodes decrease, but these approximate symmetries break upon further training (Fig. \ref{fig:results-formation-and-pruning}e) due to exploration incentives in the reinforcement learning algorithms (Methods \ref{sec:experiments_cfg}). 

\subsubsection*{Synchronization captures regularity in the data}

The activity synchronization theorem, Eq. (\ref{eq:synch-activity}),
establishes that the nodes of a fiber have the same activity for all
inputs. We ask how node synchronization
changes when we limit the input space and, in particular, if we draw
inputs from individual classes in a classification task. We find that
the correlation matrices for node activity with data drawn from 
individual classes form large clusters of
synchronized nodes (Fig. \ref{fig:clusters_class}).  These class-conditional
clusters are a coarsening of the partitioning into fibers (see Methods
\ref{material:fibers-clusters} - Fig. \ref{fig:scores}). Thus,
each class-conditional synchrony cluster can be decomposed into
multiple fibers, forming hierarchical clusters of nodes. The increasing size of fibers across layers is paralleled by the observed node synchronization (Fig.~\ref{fig:clusters_class}), which is more prevalent in later layers of the network. 

One can interpret the class-conditional synchronization clusters  in terms of features of the input: It is well accepted that node activity in hidden layers captures a hierarchy of features in the stimulus, that are shared across classes; e.g. simple examples for 2D CNN are edges, junctions, and corners
\cite{zeiler2014visualizing}. The intersection of all class-conditional clusters, capturing all the shared features, is embodied in the fibers. The increasing size of fibers across layers captures the hierarchical organization of features.  Therefore, fibers capture the hierarchical regularity present in all training data. In the course of training, these clusters of synchronization emerge in a hierarchical process (Fig.~\ref{fig:sync}a,b) consistent with the Coarse-Graining Theorem (Methods \ref{sec:hierarchical}). 

\subsubsection*{Fibration and covering compression versus performance}

Once fibers have formed, we can compress the network. As we relax the thresholds $\varepsilon_{\rm fib}$ and $\varepsilon_{\rm op}$ a larger number of nodes can be grouped into quasi-fiber and -opfiber, and we achieve stronger network compression. In practice, we have found that compression with the same value $\varepsilon$ produces different loss changes $\Delta L$ in different layers, consistent with the dependence of $K^{(\ell)}_c$ in Eq. (\ref{eq:loss-change-bound}). We therefore adopted layer-dependent thresholds and developed efficient search algorithms to find the optimal set of $\varepsilon_{\rm fib}^{(\ell)}, \varepsilon_{\rm op}^{(\ell)}$ to achieve maximum compression with a fixed budget $\Delta L = \varepsilon_{\rm Loss}$ (see Methods \ref{si:loss-threshold-search}). 

\begin{figure}[htbp]
\centering
\includegraphics[scale=0.20]{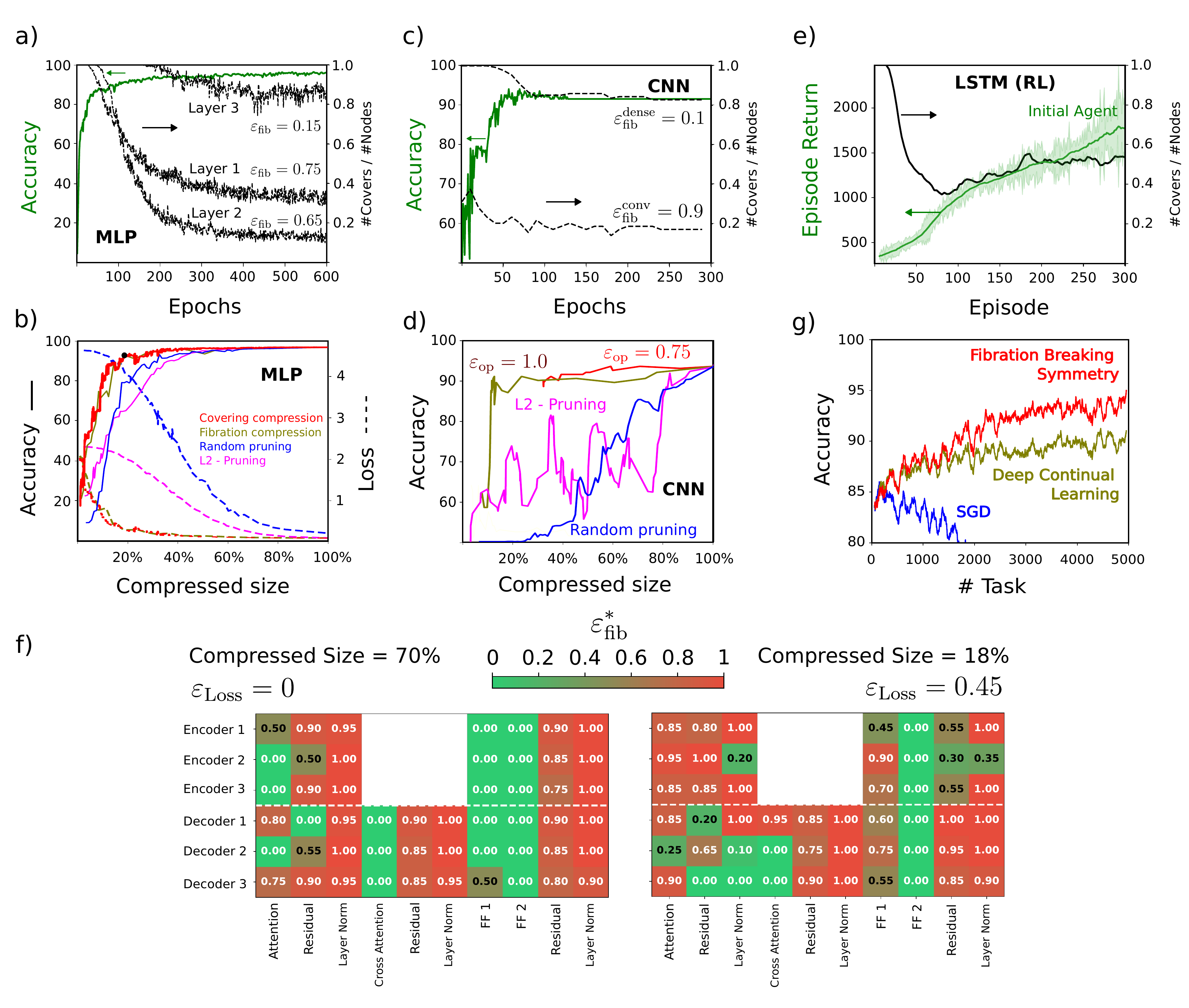}
\caption{
{\bf Empirical results.} 
{\bf (a)} MLP trained on MNIST dataset. Classification accuracy increases across learning epochs (green), while the number of trivial (single-node) quasi-covers ($\varepsilon_{\rm op} = 1$) decreases in favor of larger multi-node quasi-covers (black dashed curves). 
{\bf (b)} Accuracy of the compressed network depends on the level of compression with fibration compression (olive) and covering compression (red).  Alternative pruning methods at different levels of compression (blue and pink curves). Dashed curves indicate the change in loss. 
{\bf (c)} Same as (a) but for CNN trained on ImageNet. 
{\bf (d)} Same as (b) for CNN. Red and olive curves correspond to optimal covering compression with different values of $\varepsilon_{\text{op}}$.
{\bf (e)} Agent with LSTM trained via reinforcement learning to play the Atari game Beam Rider. Returns increase with learning episode (green curve). Black curve shows the proportion of quasi-covers ($\varepsilon_{\rm fib} = 0.5, \varepsilon_{\rm op} = 1.0$) as a function of training episode. In panels (a)-(c)-(e) arrows point to the axis associated with each curve.
{\bf (f)} Compression of a sequence-to-sequence Transformer trained for German-English translation. Optimal $\vec{\varepsilon}^{*}_{\rm fib}$ obtained for each bock in the network architecture. Red indicated stronger compression, green, no compression. Tolerance $\varepsilon_{\rm Loss}=0$ accepts no change in loss (left) or a threshold value $\varepsilon_{\rm Loss}=0.45$ that will not affect translation task performance (right) (see Fig. \ref{fig:metrics_MLP_CNN}c, right).
{\bf (g)} Continual learning with sequential ImageNet binary classification tasks, using conventional SGD (blue),  Continual Backpropagation (brown), and the proposed FSB (red).
}
\label{fig:results-formation-and-pruning}
\end{figure}
{\bf Compression in MLPs:} For a three-layer MLP trained in MNIST, the optimal covering compression (Fig. \ref{fig:metrics_MLP_CNN}a, right, red curve) achieves substantial compression of down to 17\% of the original network size with no change in loss or accuracy (Fig. \ref{fig:results-formation-and-pruning}b, red dashed and solid curves, respectively). We obtain similar results when we limit the search to fibration compression (i.e. we search for three optimal $\varepsilon_{\rm fib}^{(\ell)}$ while keeping $\varepsilon_{\rm op}=1$, Fig. \ref{fig:metrics_MLP_CNN}, right, olive).  We find that covering compression (\ref{fig:results-formation-and-pruning}b, red) tends to outperform fibration compression (Fig. \ref{fig:results-formation-and-pruning}b, olive).  We compared this with alternative compression methods \cite{meng2020pruning} such as randomly pruning nodes (blue curve) and pruning nodes with small weights (L2-norm, pink curve). Fibration and covering compression identify and merge nodes that play identical functional roles and thus have no impact on loss. In contrast, these conventional pruning methods remove potentially unique contributions and therefore immediately impact the loss. 

{\bf Compression in CNNs:} Similar results are observed for a CNN trained on ImageNet (Fig. \ref{fig:results-formation-and-pruning}d). The red curve represents covering compression with all $\varepsilon_{\rm op} = 0.75$. The olive curve, in turn, corresponds to $\varepsilon_{\rm op} = 1.0$, which is equivalent to fibration compression. Again, we find that covering compression outperformed fibration compression as well as conventional pruning methods.  Here we found the optimal fibration thresholds at a given opfiber threshold (see Fig. \ref{fig:metrics_MLP_CNN}, right). 

{\bf Compression in Transformers:} We also tested compression in a sequence-to-sequence Transformer trained for German-English translation (Methods~\ref{sec:experiments_cfg}) using the same data and architecture as in previous work \cite{vaswani2023attention}. With three encoder and three decoder layers, the network has 41 different blocks, each with its own $\varepsilon_{\rm fib}$ (\ref{fig:results-formation-and-pruning}f). We only tested fibration compression (with fixed $\varepsilon_{\rm op} = 1$). The optimal values of ${\varepsilon}^{*}_{\rm fib}$ are shown in the matrices of Fig. \ref{fig:results-formation-and-pruning}f. Note that $\varepsilon_{\rm fib} = 1$ implies strong compression, while values close to 0 imply weak compression. Despite a strict limit to the change in loss ($\varepsilon_{\rm Loss}=0$), we can compress the network to $70\%$ of its size (Fig. \ref{fig:results-formation-and-pruning}f, left). However, we can tolerate some increase in loss without affecting translation performance (Fig. \ref{fig:metrics_MLP_CNN}c, left). At this tolerance ($\varepsilon_{\rm Loss}=0.45$) the transformer can be compressed to $18\%$ of its original size (Fig. \ref{fig:results-formation-and-pruning}f, right).

\subsubsection*{Fibration symmetry breaking for continual learning}

Deep networks have been found to lose their ability to learn when trained sequentially on multiple new tasks \cite{lyle2023understanding, lyle2024disentangling}. We evaluate this on the continuous ImageNet task \cite{vandeven2022three} (Methods \ref{material:resultsDCL}). First, we observe that conventional SGD stops learning after approximately 2,000 tasks (Fig. \ref{fig:results-formation-and-pruning}g, blue curve), reproducing the results of \cite{dohare2024loss}. We find that this stagnation is accompanied by cover formations (Fig. \ref{fig:distrib_fibers}). 
The emergence of symmetry induces parameter tying, which reduces the degrees of freedom in the parameter space. This constraint directly explains the observed loss of network plasticity: with fewer independent directions for optimization, learning stagnates.

Fibration symmetries that emerge during SGD reduce the network's degrees of freedom and manifest as redundant, synchronized activity that limits the repertoire of learnable features. While several ad-hoc heuristics attempt to promote diversity (see Discussion), they operate indiscriminately by preventing or breaking the fibers that naturally emerge from the data with SGD.

In contrast, we propose {\em Fibration Symmetry Breaking} (FSB), a principled and targeted approach to symmetry breaking \cite{anderson1972more} by first compressing covers and then randomly choosing the remaining weights 
as follows (Fig. \ref{fig:hierarchy}e):
\begin{enumerate}  
\item Prune: Identify covers and reduce the graph to the covering base by
  applying the fibration compression rule
  Eq. (\ref{eq:fib-weight-compression}) to ensure that the learned
  function of the network and its performance are preserved.
\item Sprout: Add naive nodes that restore the original size. Each new node is
  initialized with random input and output weights (with zero
  outgoing weights to the nodes in the base). This restores degrees of
  freedom for further learning without interfering with the learned
  base.
\end{enumerate}

We compared performance with the state-of-the-art ``Continual Backpropagation'' \cite{dohare2024loss}, which has been shown to outperform existing heuristics such as shrink-and-perturb \cite{chebykin_shrink-perturb_2023}. Continual backpropagation resets the nodes that have become inactive. In Methods \ref{material:DCL}, we show that this technique is a special case of our symmetry-breaking mechanism. In the sequential ImageNet task (Fig. \ref{fig:results-formation-and-pruning}g) continual backpropagation (green) and our method FSB (red) are comparable up to task 1,500. However, beyond that point, our method outperforms continual backpropagation, cutting the remaining deficit in accuracy by half (from 90\% to 95\% by task 5,000). This indicates that the FSB continues to capture useful features that generalize between classes, outperforming the state of the art.

\section{Discussion}

Our work frames deep learning not as a process of brute-force parameter-tuning but as a fundamental mechanism of structure formation. Through this lens, optimization naturally aligns a network's internal topology to the latent geometry of the data, isolating a compact set of features to describe the data. This framework provides a conclusive, mathematically rigorous explanation for long-observed empirical phenomena that have previously lacked a unified theoretical basis. For instance, the spontaneous emergence of fibration symmetry explains ``node collapse" and the generation of redundant, synchronized neural activity, particularly within the deeper layers of over-parameterized models \cite{papyan2020prevalence, doimo2022redundant, velarde2024role}. We show that stochastic gradient descent actively stabilizes these covering symmetries, acting similarly to noise-induced stabilization in complex dynamical systems \cite{herzog2015noise}. We propose that fibration symmetries emerge from a hierarchical loss landscape (Methods \ref{sec:hierarchical}). By demonstrating that networks navigate toward increasingly coarse symmetry structures during optimization, this geometric perspective resolves the classic paradox of over-parameterization: gradient descent finds structured symmetric solutions where the emergent model is dramatically simpler than the raw parameter counts suggest.

Beyond explaining the learning dynamics, this framework unifies a wide array of seemingly disparate ad-hoc AI heuristics under a single theoretical view. Hard-coded inductive biases, such as weight-tying in CNNs, can now be understood as simply pre-defined graph symmetries; our results demonstrate how such constraints can instead emerge organically through learning. Crucially, our theory explains why a variety of popular empirical techniques successfully enhance network performance by actively breaking or preventing these symmetries. For example, dropout \cite{srivastava_dropout_2014}  breaks symmetries by updating only a fraction of nodes selected at random, while residual connections \cite{he2016deep} break symmetries that would otherwise develop within the skipped layers (Methods~\ref{material:hypergraphs}). Similarly, methods designed to explicitly avoid representation redundancy \cite{wang2021convolutional} directly prevent fiber formation; this includes structural redundancy reduction algorithms \cite{zbontar2021barlow} (Methods \ref{si:orthogonal}), feature recombination \cite{qiu2021slimconv}, or minimization for cross-correlation  \cite{zbontar2021barlow} (Methods \ref{si:orthogonal}). While these existing approaches operate blindly, our symmetry-driven pruning offers an exact, theory-motivated methodology for model compression. Because fibration compression merges functionally redundant nodes rather than pruning unique computations, it preserves the exact function of the model, compressing models to 17-18\% or their original size. This structural compression complements downstream techniques such as weight quantization \cite{sharma2023truth,tomut2024compactifai} or linear low-rank matrix decompositions (Methods \ref{sec:compression-SVD}).

Importantly, our framework directly resolves the critical trade-off between stability and flexibility in lifelong learning systems. The emergence of covers under standard SGD induces parameter tying, which restricts optimization and directly causes the well-documented loss of network plasticity over time \cite{lyle2024disentangling}. Rather than applying indiscriminate resets to final layers or specific nodes, which often fail to fully overcome this stagnation, our Fibration Symmetry Breaking protocol offers a surgical alternative: it compresses the emergent covers to safeguard past knowledge and injects randomized, naive nodes to reclaim lost degrees of freedom without disrupting the learned base. The resulting 50\% reduction in the performance gap on sequential ImageNet benchmarks compared to state-of-the-art continuous backpropagation \cite{dohare2024loss} establishes FSB as a highly plausible foundation for true lifelong learning systems. 

Several theoretical and practical horizons remain open for investigation. First, establishing the optimal sequencing to combine fibration compression with low-rank  and quantization pipelines represents an immediate engineering frontier. Second, while our current approach relies on static model compression after training, developing algorithms for dynamic  compression during training may allow stronger compression factors, or  conversely, continual learning with larger datasets. Finally, while our framework outlines how optimization trajectories navigate downhill through a tree of increasingly coarse-grained symmetry quotients, the relation of this to the hierarchical structures in the data itself remains to be more strictly formalized.

In conclusion, our work frames learning not as parameter-tuning, but as structure formation. The process transforms the opaque black box into an interpretable colored graph, where emergent symmetries reveal the data's learned regularities. This theory-driven perspective offers a path to designing future AIs in which mathematical principles, not brute-force scaling,  drive performance, generalization, and interpretability.

\section{Methods}
\localtableofcontents

\subsection{Global and local symmetries in graphs: coverings, fibrations, and opfibrations}
\label{material:symmetries}

We first discuss symmetries in a binary (unweighted) graph where connections are restricted to binary values (0 or 1) to introduce the main concepts didactically, following Fig. \ref{fig:hierarchy}.
Graph fibrations were first introduced for binary graphs by Boldi and Vigna \cite{boldi2002fibrations} based on the categorical definition of Grothendieck fibration from \cite{grothendieck1959technique}.
Methods \ref{material:symmetry-ffn}
generalizes the definition to weighted graphs and hypergraphs for applications to DNNs.

A directed graph $G=(N_G,A_G)$ consists of a set $N_G$ of nodes and a set $A_G$ of connections. Each connection is associated with a source and a target node, and one way to represent the full structure of the graph is its adjacency matrix. For every node $u \in N_G$, there is a corresponding input tree $T_u$ that represents the set of all paths of $G$ ending in $u$. Similarly, there exists a corresponding output tree $\hat{T}_u$ that represents the set of all paths of $G$ starting from $u$. We say that two input trees $T_u$ and $T_v$ are isomorphic ($T_u \sim T_v$) when there is a bijective map $\tau_{u\to v}: T_u \rightarrow T_v$, which maps the nodes and connections of $T_u$ one-to-one to the nodes and connections of $T_v$. The same definition and notation apply to output trees with a bijective map denoted by $\hat{\tau}_{u\to v}: \hat{T}_u \rightarrow \hat{T}_v$.

For graph $G$,
\begin{enumerate}
    \item A \textit{fibration symmetry} of $G$ is a surjective homomorphism $\varphi_{\rm fib}: G \to B$ that preserves the input tree:
    \begin{equation}
        \label{eq:fibration}
        \forall u,v \in N_G: \varphi_{\rm fib}(u)=\varphi_{\rm fib}(v) \in N_B \iff T_u \sim T_v.
    \end{equation}
    \item An \textit{opfibration symmetry} of $G$ is a surjective homomorphism $\varphi_{\rm op}: G \to B$ that preserves the output tree:
    \begin{equation}
        \label{eq:opfibration}
        \forall u,v \in N_G: \varphi_{\rm op}(u)=\varphi_{\rm op}(v) \in N_B \iff \hat{T}_u \sim \hat{T}_v.
    \end{equation}
    \item A \textit{covering symmetry} of $G$ is a surjective homomorphism $\varphi_{\rm cov}: G \to B$ that preserves the input and output trees:
    \begin{equation}
        \label{eq:covering}
        \forall u,v \in N_G: \varphi_{\rm cov}(u)=\varphi_{\rm cov}(v) \in N_B \iff T_u \sim T_v \qquad \& \qquad \hat{T}_u \sim \hat{T}_v.
    \end{equation}
    
    \item  An {\it automorphism} of a graph $G$ is a bijective map $\pi_{auto} : G \to G$ , such that the pair of nodes $u$ and $v$ forms an connection $(u, v)$ if and only if ($\pi_{auto}(u) , \pi_{auto}(v)$) also forms an edge. Automorphisms are also called permutation symmetries.
\end{enumerate}

Each symmetry, fibration, opfibration, covering, and automorphism,  gives rise to partitions called (respectively): fiber $\mathcal{C}^{\rm fib}$, opfiber $\mathcal{C}^{\rm op}$, cover $\mathcal{C}^{\rm cov}$, and orbit partitions $\mathcal{C}^{\rm auto}$, which are hierarchical in the sense of coarsening: $\mathcal{C}^{\rm fib} \ge \mathcal{C}^{\rm cov} \ge \mathcal{C}^{\rm auto}$
and $\mathcal{C}^{\rm op} \ge \mathcal{C}^{\rm cov} \ge \mathcal{C}^{\rm auto}$.

Automorphisms form symmetry groups. That is, the set of automorphisms of a graph satisfies the composition law, associativity, and has an inverse and identity. They are global symmetry transformations of the graph, as they apply a global permutation of all nodes that preserves the global adjacency matrix; that is, the connections between all nodes remain the same before and after the permutation. 

Two nodes are in an orbit iff:
\begin{equation} 
u \underset{\rm auto}{\sim} v  \iff \exists \, \pi_{auto}\in \text{Aut}(G) : \pi_{auto}(u) = v,
\end{equation}
where $\text{Aut}(G)$ is the symmetry group of $G$. The orbits form the orbital partition of the graph.

A partition $\mathcal{C}$ of a graph $G$ is \textit{finer} than (also called \textit{a refinement of}) a partition $\mathcal{C}'$ 
if every element of $\mathcal{C}$ is a subset of some element of $\mathcal{C}'$.
That is, $\mathcal{C}$ is a further fragmentation of $\mathcal{C}'$.
We denote this case as $\mathcal{C} \le \mathcal{C}'$.
In this example, $\mathcal{C}'$ is \textit{coarser} than  $\mathcal{C}$ ($\mathcal{C}' \ge \mathcal{C}$) and $\mathcal{C}'$ is a {\textit coarsening} (merger) of elements of $\mathcal{C}$.
The coarsest possible partition of a graph is one in which all nodes belong to a single partition. This is a graph with maximal symmetry. The finest possible partition is the (trivial) partition into singletons, which is the partition of a graph with no symmetry at all (or with only trivial symmetry).

As seen in Fig. \ref{fig:hierarchy}, for a directed graph, the fiber and the opfiber partitions are two different coarsenings of the cover partition, and the cover partition is a coarsening of the orbit partition. In principle, there is no relation (finer or coarser) between the fiber and the opfiber partition. However, when the graph is undirected, the fiber, opfiber, and cover partitions coincide. Yet, they are still a coarsening of the orbital partition.
Furthermore, as shown in Fig. \ref{fig:sync}b, the cluster synchronization partition (obtained dynamically by clustering activity synchronization in the forward pass, averaging over samples) is a coarsening of the fiber partition (but not of the opfiber partition).

All four partitions, fibers, opfibers, covers, and orbits, are different balanced colorings of the graph \cite{golubitsky2006nonlinear}:
fibers are in-balanced colorings, opfibers are out-balanced colorings, and covers are in- and out-balanced colorings. Orbits are also in- and out-balanced colorings. 
However, orbits have an extra condition compared to covers: two nodes in an orbit must be obtainable from one another through the application of a permutation symmetry (automorphism) of the graph.

For any partition, we will use the notation $c$ to indicate some color and $|c|$ the cardinality of the color, i.e., the number of nodes with the same color $c$.

Due to the refinement relations, the number of balanced colors increases from the (op)fiber partition to cover to orbit partition. This is reflected in Fig. \ref{fig:hierarchy} with 9 colors for fibers and 9 colors for opfibers as compared to 10 colors for covers and 12 colors for orbits (consider that the blank nodes in the figure have all different colors). This also reflects the increase in symmetry from automorphisms to covers to (op)fibers.

We show an example of an automorphism that maps 1 to 2, 2 to 1, 3 to 4, 4 to 3, 5 to 6, 6 to 5, and all other nodes map to themselves in  Fig.~\ref{fig:hierarchy}d. It is denoted in cycle notation by:
\begin{equation}
\pi = (1 \, 2) (3 \, 4) (5 \, 6).
\label{eq:legal_perm1}
\end{equation}
The other nodes are trivially permuted.

Fibrations, opfibrations, and coverings strictly generalize automorphisms by relaxing the global constraints to impose only a local preservation of structure, such as the input and/or output trees. Examples of these symmetries in a graph are shown in Fig.~\ref{fig:hierarchy}. 
These symmetries do not form groups; they form categories. The color-preserving isomorphisms between (output)input trees form groupoids (i.e., groups without a composition law) \cite{golubitsky2006nonlinear}. The fibration framework is therefore also known as 'the groupoid formalism' in the work of Golubitsky and  Stewart \cite{golubitsky2006nonlinear}.

Remark  1. One might wonder why coverings are not automorphisms, since both preserve inputs and outputs. The key distinction is scope: coverings preserve the input and output structure of the nodes within each cover, while automorphisms preserve the input and output structure globally across the entire graph. 

To exemplify the meaning of input and output trees and its relation with automorphism, we consider the graph $G$ in Fig.~\ref{fig:hierarchy}c and call the remaining nodes as: first hidden layer:  7 (top blank) and 8 (bottom blank), second hidden layer: 9 (top red), 10 (bottom red), third hidden layer: 11 (top blue), 12 (bottom blue). We also call 13 and 14 the top and bottom inputs, and 15 the graph's output.
Red nodes 9 and 10 form a cover since they have isomorphic input and output trees. That is, there is an input-isomorphism: 
$\tau_{9\to 10} = (9, 7, 13, 14) \rightarrow (10, 8, 13, 14)$ and an output-isomorphism: $\hat{\tau}_{9\to 10} = (9, 11, 12, 15) \rightarrow (10, 11, 12, 15)$ from the input and output tree of node 9 to 10. There is also an isomorphism between connections, but since all connections are the same (binary graph), there is no need to specify it.

One may wonder if nodes 9 and 10 might form an orbit, i.e., if they are also symmetric under a permutation symmetry. If we permute  9 and 10, then 7 and 8 need to be permuted as well to preserve the inputs of 9 and 10 (11 and 12 do not need to be permuted, as the outputs of 9 and 10 are preserved by the permutation). Thus, the permutation $\pi = (9 \, 10) (7 \, 8)$, in principle, preserves the adjacency of 9 and 10. However, node 8 is now connected to 3 and 4, while node 7 is not, and there is no other permutation that can restore the original adjacency (for instance, permuting 3 and 4 will not fix it). Thus, while $\pi$ is a valid permutation of the graph, it is not a permutation symmetry (automorphism) of the graph. The violation occurs not in the input or output trees of 9 and 10. In fact, one can check that the input and output trees of nodes 9 and 10 are still isomorphic in the permuted graph $\pi(G)$. The violation occurs in other parts of the graph: the inputs of nodes 3 and 4. This simple example illustrates the 'local' versus 'global' preservation condition for (op)fibrations/coverings versus automorphisms, an important ingredient of the framework presented here.

Remark 2. Note that 'local' gauge symmetries in physics have a different meaning \cite{georgi2000lie}. They refer to spacetime locality captured by the fiber bundle, not the fibration-theoretic locality. Fibrations are more general than fiber bundles since they only require relaxed fibers with local symmetry, rather than the symmetry groups required by the fibers in a fiber bundle. Fiber bundles are widespread in theoretical physics but not in biology or AI (see \cite{makse2026symmetry} for further details).

Remark 3. In our formulation, nodes with no inputs (for instance, all nodes in the input layer of an MLP) are assigned different balanced colors, ie, each node belongs to a different fiber. Same for the output layer and opfibers. This differs from the original definition of graph fibrations in \cite{boldi2002fibrations}.

Remark 4. When $G$ is a graph with flavored connections \cite{boldi2006graph}, the previous symmetry definitions remain valid using colored input and output trees. Two flavored trees ($T_u$ with flavor $q_u$ and $T_v$ with flavor $q_v$) are considered isomorphic if there exists a bijection $\tau: T_u \rightarrow T_v$ that is a one-to-one mapping of both nodes and connections from $T_u$ to $T_v$ and preserves flavors of the connections, i.e.  $\forall a\in A_{T_u}: q_u(a) = q_v \circ \tau(a)$. When all connections have the same flavor, the initial definitions are recovered as a special case.

\subsection{Lifting property by fibration symmetry: Construction of fibration base}
\label{material:lifting-prop}

So far, we have not provided a complete definition of a base
$B$ for any symmetry because Eqs. (\ref{eq:fibration}, \ref{eq:opfibration}, \ref{eq:covering}) only give us information about $N_B$ but not about the connections $A_B$. The connections for base fibrations and opfibrations are defined using \textit{lifting property}. In the case of fibration symmetry, this \textit{lifting property} is based on the structure of the in-neighborhoods of the graph $G$. The connections in the case of binary connections are selected as follows:

\begin{enumerate}
\item Calculate the partition $\mathcal{C}^{\rm fib}$ (see colors in Fig. \ref{fig:hierarchy}a).
\item Identify the in-neighborhood for all nodes of $G$.
\item Replace the node IDs in the in-neighborhoods with its color to obtain the in-neighborhoods of the nodes in the fibration base $B_{\rm fib}$.
\item Connect the nodes of $B_{\rm fib}$ using information from the neighborhoods. Note that each node in the base $B$ can receive more than one connection from another (see the base in Fig. \ref{fig:hierarchy}a). 
\end{enumerate}

Multiple connections indicate that the input is the addition of single-connection inputs. For networks with weighted connections, this addition is reflected in the sum of Eq.~(\ref{eq:fib-weight-compression}) derived in Methods~\ref{sec:fiber-compression-MLP}.

For opfibration symmetry, the lifting property is analogous but differs in that it is based on the structure of the out-neighborhoods of the graph $G$. The lifting property ensures that the input tree (resp., output) of a node in the fibration (opfibration) base is isomorphic to the trees of the nodes in that fiber (opfiber).

For a covering symmetry -- since it is simultaneously a fibration and an opfibration symmetry -- the base can be constructed using either the fibration or the opfibration lifting property after calculating the partition $\mathcal{C}^{\rm cov}$. Both cases are shown in Fig. \ref{fig:cover_symm_compression_two_ways}. In general, the results will differ. Therefore, one has to select either to follow a fibration compression (preserving the inputs, as we have done in Fig. \ref{fig:hierarchy}e) or an opfibration compression. 

In this work, we will use the construction of the cover base via the fibration lifting property, as it preserves the forward computation of the deep networks. This should not be confused with the proper fibration symmetry that compresses the fibers using the fibration map shown in Fig. \ref{fig:hierarchy}a.

\begin{figure}[ht!]
\centering
\includegraphics[scale=0.2]{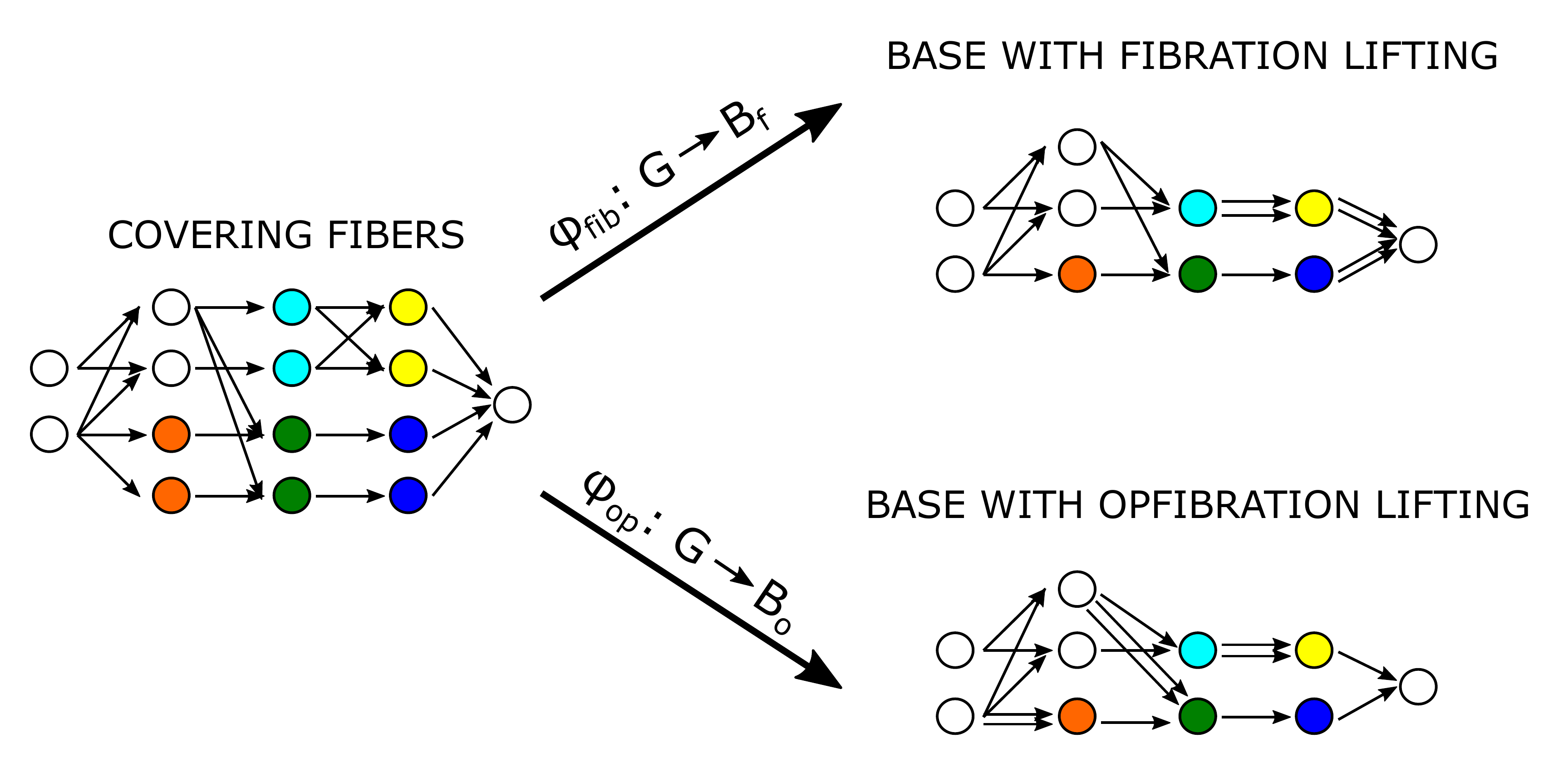}
\caption{\textbf{Compression to the cover base:} One can preserve the input tree or output tree, but generally not both unless all covers have the same cardinality which is not satisfied in this case, since 5 covers have cardinality 2 and 5 covers have cardinality 1 (trivial). Note that our notation uses white for trivial covers. A more strict representation should assign a different color to every one of the trivial covers, but this representation would be harder to grasp with so many colors.}
\label{fig:cover_symm_compression_two_ways}
\end{figure}

\subsection{Fibration lifting operation for weighted network}
\label{material:lifting}

Here, we elaborate on the concept of fibration lifting for a weighted network. The origin of the name ``lifting'' is the inverse operation of compression in Fig.\ref{fig:hierarchy}, where the base is ``lifted'' to a full graph. Since compression is not an injective function, the lifting is not unique. For example, the following weights $W$ (full graph) are transformed into the same weights $\Tilde{W}$ (base),
\begin{equation}
    W^{(\ell)}_{ij} = \frac{1}{|c(j)|} \hat{W}^{(\ell)}_{c'(i)c(j)} + \Delta_{ij}
\end{equation}
where $\sum_{\substack{k \in c'\\m \in c}} \Delta_{km} = 0$. 

In Fig. \ref{fig:lifting}, we compress using Eq. (\ref{eq:fib-weight-compression}) from $G_1$ to the base $B'_1$; then, $B'_1$ to the minimal base $B$. $B$ is a base of $G_1$ that has two fibers (blue and red). In $B'_1$, two nodes of the blue fiber are compressed; while in the minimal base $B$, the three nodes of the blue fiber are compressed. Two networks in different configurations can have the same base. For example, $G_1$ and $G_2$ can be compressed to $B'_1$, while $G_3$ and $G_4$ can be compressed to $B'_2$. In the end, these four graphs have the same minimal base $B$. In this work, we address the case of weighted graphs. With binary connection weights (representing simply the connection or its absence), the scenario is reduced to that described in \cite{boldi2002fibrations}.

All networks $G_1,G_2,G_3,G_4$, bases $B_1,B_2$, and the minimal base $B$ share the same forward computation.  That is, for every input $x$, the output $\hat{y}$ is identical across all networks. Moreover, for every pair $(x,y)$ in the dataset, the loss values are the same for all networks. This implies that the loss function $\mathcal{L}(w)$ is the same for all aforementioned networks (a degeneracy of the loss function). More precisely, if two networks $G$ and $G'$ have the same minimal base $B$, then $\mathcal{L}(w_G) = \mathcal{L}(w_{G'})$ where $w_G$ and $w_{G'}$ denote the parameters of $G$ y $G'$, resp. This result has important implications for the loss landscape as discussed in Methods \ref{sec:hierarchical} and Fig. \ref{fig:sync}d. 

\begin{figure}[ht]
    \centering
\includegraphics[scale=0.25]{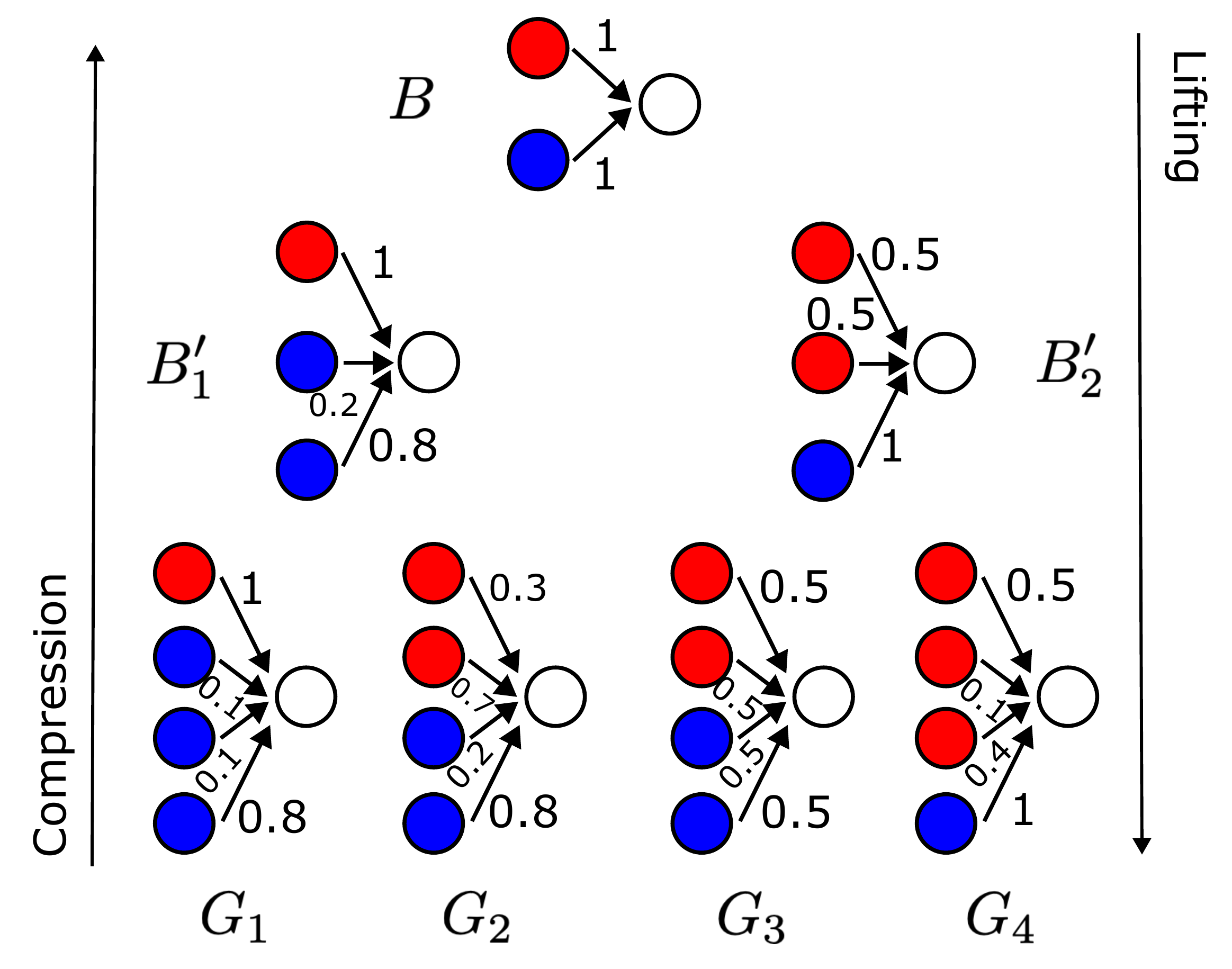}
    \caption{Examples of compression and lifting operations as reduction and expansion of dimensions, respectively. Colors represent fibers.}
    \label{fig:lifting}
\end{figure}

\subsection{Fibration symmetry implies activity synchronization}
\label{sec:fibration-symmetry-activity-synchronization}

Here we formally generalize the concept of fibration symmetry from binary weights to continuous weights. Two nodes $i$ and $j$ in layer $\ell$ belong to the same fiber (denoted $i\underset{\rm fib}{\sim}j$) if and only if (iff):

\begin{itemize}
    \item \textbf{Fibration symmetry:} 
    \begin{equation}
    i \underset{\rm fib}{\sim} j \iff \forall c \in \mathcal{C}^{\rm fib}_{\ell-1} : \quad \sum_{k \in c} W^{(\ell)}_{ik} = \sum_{k \in c} W^{(\ell)}_{jk}
\label{eq:fibration-definition}
\end{equation}
    
\end{itemize}
This criterion partitions the nodes in the layer $\ell$ into a
coloring $\mathcal{C}^{\rm fib}_\ell$ based on the input weights
$W^{(\ell)}_{ik}$ and the color partitioning of the previous layer
$\mathcal{C}^{\rm fib}_{\ell-1}$. In words, this definition states
that two nodes are in the same fiber iff for all colors from the
previous layer the sums of weights from these colors are the same. Or, more simply, the total input from each color is the same, which also ensures that their activity is the same. This recursive definition starts in the input layer, $\ell=0$, with all nodes ($d_0$ nodes) having distinct colors $\mathcal{C}^{\rm fib}_0 = \{1,2,...,d_0\}$, i.e., trivial fibers.  The definition of fibers in Eq. (\ref{eq:fibration-definition}) based on the sum of weights is the correct generalization from binary graphs (Fig. \ref{fig:hierarchy}) to weighted graphs, as we show in Methods \ref{material:symmetry-ffn}. 

Similarly to symmetry-induced invariance theorems in physics \cite{wilczek2016beautiful},  fibration symmetry induces activity synchronization during forward inference. This is the subject of the following theorem:

\begin{itemize}
    \item \textbf{Activity synchronization:}
     Given two nodes $i$ and $j$ in
layer $\ell$, the following hold in networks with inputs $x$:
    \begin{equation}
        i \underset{\rm fib}{\sim} j \implies h^{(\ell)}_i = h^{(\ell)}_j \qquad \text{for any network input }x
    \label{eq:synch-activity}
    \end{equation}
\end{itemize}

The proof of Eq. (\ref{eq:synch-activity}) is in Methods \ref{material:proof-synchronization}. The relation emerges naturally from the recursive dependence of neuronal activity $h_i^{(\ell)}$ on activity in the previous layer through input weights $W^{(\ell)}_{ik}$. Similar proofs exist for dynamical systems in
\cite{makse2026symmetry,deville2015modular,nijholt2016graph}.

Equation (\ref{eq:mlp-activity}) establishes the relationship between the network structure captured by its weights and the network function captured by its activity.  Therefore, the input tree structure governs the synchronization patterns.

\subsection{Input and output trees in weighted feedforward networks}
\label{material:symmetry-ffn}

In weighted feedforward networks (e.g. Multi-Layer Perceptrons), an automorphism symmetry between two nodes would require that they are in the same layer and that they have identical weight vectors at the input and output, so that permuting the two nodes keeps the entire network unchanged. Fibration and opfibration symmetries are less restrictive than automorphisms. As we shall see, they specify a weaker constraint on the weights. 

In order to define fibers in weighted layered graphs, we first discuss isomorphic input trees. The input tree $T_i$ of a node $i$ in layer $\ell$ consists of all paths from the input layer to node $i$. For example, for a node $i$ in the first layer ($\ell=1$), its input tree is the vector 
\begin{equation*}
    T^{(1)}_i = [W^{(1)}_{i1},W^{(1)}_{i2},...,W^{(1)}_{id_0}].
\end{equation*} 
Then, two nodes have isomorphic input trees if $T^{(1)}_i = T^{(1)}_j$. Using this criterion, we can define fibers in the first layer. For layer $\ell=2$, the input tree $T^{(2)}_i$ of node $i$ can be expressed as the concatenation of the input trees $T^{(1)}_m$ of nodes $m$ in layer $\ell=1$, and their corresponding connection weights $W^{(2)}_{im}$, i.e. 
\begin{equation*}
T^{(2)}_i = [T^{(1)}_1, W^{(1)}_{i1}, T^{(1)}_2, W^{(1)}_{i2}, ... ,T^{(1)}_{d_1}, W^{(1)}_{id_1}].
\end{equation*} 
If the fibers $c$ in the first layer are already known, $T^{(2)}_i$ can be represented more compactly as the concatenation of the input trees of the fibers $T^{(1)}_c$ (because they are all the same within a fiber), but we do have to aggregate the weights as $\sum_{k \in c} W_{ic}$ to replicate their sum of effects in layer $\ell=1$. Therefore, two nodes $i$ and $j$ of the layer $\ell=2$ have isomorphic input trees, $T^{(2)}_i = T^{(2)}_j$,  iff $\forall c: \sum_{k \in c} W_{jk} = \sum_{k \in c} W_{ik}$. This criterion therefore defines the fiber for layer $\ell=2$. This argument can be extended to the next layers $\ell=3,...,N$ thus recursively defining fibers.  This leads to the equal-sum criterion in the definition (\ref{eq:fibration-definition}) of fiber in weighted graphs. Similar formulations have been used in the context of dynamical system synchronization \cite{aguiar2017patterns,leifer2022symmetry}. 

In order to define flavored opfibers, we will now discuss isomorphic flavored output trees. (They are referred to as colored weights in \cite{boldi2006graph}, but here we use the word ``flavor" to not confuse with node colors).  For a node $i$ in layer $\ell=N-1$, its output tree is the concatenation of the output weights $w$ of $i$ with its corresponding flavors $q(w)$; i.e., 
\begin{equation*}
 \hat{T}^{(N-1)}_i = [(W^{(N)}_{1i}, q(W^{(N)}_{1i})),(W^{(N)}_{2i},q(W^{(N)}_{2i})),...,(W^{(N)}_{d_Ni},q(W^{(N)}_{d_Ni}))]   \, .
\end{equation*}
We are interested in the case where the flavors of the connections $W_{ki}$ depend only on $i$ (i.e. $q(W_{ki}) = q(i)$). That means 
\begin{equation*}
\begin{split}
 \hat{T}^{(N-1)}_i &= [(W^{(N)}_{1i}, q(i)),(W^{(N)}_{2i},q(i)),...,(W^{(N)}_{d_Ni},q(i))] = \\
 &=([W^{(N)}_{1i}, W^{(N)},...,W^{(N)}_{d_Ni}],q(i)) \,\, .   
\end{split}
\end{equation*}
Then, two nodes have isomorphic flavored output trees if $q(i) = q(j)$ and $\forall k: W^{(N)}_{ki} = W^{(N)}_{kj}$. Using this criterion, we can define flavored opfibers in layer $\ell=N-1$. For layer $\ell=N-2$, the output tree $\hat{T}^{(N-2)}_i$ of node $i$ can be expressed as the concatenation of the output trees $\hat{T}^{(N-1)}_m$ of nodes $m$ in layer $\ell=N-1$, and their corresponding connection weights $W^{(N-1)}_{mi}$, i.e., 
\begin{equation*}
\hat{T}^{(N-2)}_i = ([\hat{T}^{(N-1)}_1, W^{(N-1)}_{1i}, \hat{T}^{(N-1)}_2, W^{(N-1)}_{2i}, ... ,\hat{T}^{(N-1)}_{d_{N-1}}, W^{(N-1)}_{d_{N-1}i}],q(i)).   
\end{equation*}
If the opfibers $\hat{c}$ in layer $\ell=N-1$ are already known, $\hat{T}^{(N-2)}_i$ can be more compactly represented as the concatenation of the output trees of the opfibers $\hat{T}^{(N-1)}_{\hat{c}}$ (because they are all the same within an opfiber), but we do have to aggregate the weights $\sum_{k \in \hat{c}} W_{ki}$ to replicate their summed effect from layer $\ell=N-1$. Then, two nodes $i$ and $j$ of the layer $\ell=N-2$ have isomorphic output trees iff $\forall \hat{c}: \sum_{k \in \hat{c}} W_{kj} = \sum_{k \in \hat{c}} W_{ki}$ and $q(i)=q(j)$ to preserve flavors. This argument can be extended to the next layers $\ell=N-3,..., 1$. This leads to the equal-sum criterion in definition (\ref{eq:opfibration-definition}) of opfiber in weighted graphs along with the flavor-preserving condition.

To capture the coupling of forward pass and backpropagation of Eq. \ref{eq:mlp-backprop}, we flavor the weights with the coloring of the fibration symmetries (i.e., when $q(i)=c(i) \in \mathcal{C}^{\rm fib}$). This ensures that nodes in a flavored opfiber have the same $\sigma'$, not just in the current layer but for the entire tree propagating the error from the output to each node of the network. When the network is linear, $q=1$ for all connections and flavored opfibers simplify to regular opfibers, uncoupling forward from backward coloring. 

\subsection{Proof of synchronization in Eqs. (\ref{eq:synch-activity}) and (\ref{eq:synch-error})}
\label{material:proof-synchronization}

We will prove both equations by induction. Suppose that the synchronization of activity holds for the layer $\ell-1$. Let $i, j$ be two nodes in layer $\ell$ such as $i \underset{\rm fib}{\sim} j$, then: 

\begin{equation*}
\begin{split}
h_i^{(\ell)} &= \sigma \left( \sum_k W^{(\ell)}_{ik} h_k^{(\ell-1)} \right) \\
&= \sigma \left( \sum_{c \in \mathcal{C}^{\rm fib}_{\ell-1}} \sum_{k \in c} W^{(\ell)}_{ik} h_k^{(\ell-1)} \right) \\
&= \sigma \left( \sum_{c \in \mathcal{C}^{\rm fib}_{\ell-1}} h_c^{(\ell-1)} \sum_{k \in c} W^{(\ell)}_{ik} \right) \qquad \text{activity synchronization in $\ell-1$}\\
&= \sigma \left( \sum_{c \in \mathcal{C}^{\rm fib}_{\ell-1}} h_c^{(\ell-1)} \sum_{k \in c} W^{(\ell)}_{jk} \right) \qquad i \underset{\rm fib}{\sim} j \\
&= h_j^{(\ell)}.
\end{split}
\end{equation*}

Therefore, the synchronization holds for the layer $\ell$. Now, let us prove the base case of the induction ($\ell=1$).

\begin{equation*}
   i \underset{\rm fib}{\sim} j \implies W^{(1)}_{ik} = W^{(1)}_{jk} \implies \sigma \left( \sum_k W^{(1)}_{ik} x_k \right) = \sigma \left( \sum_k W^{(1)}_{jk} x_k \right) \implies h_i^{(1)} = h_j^{(1)}.
\end{equation*}

With these two steps, we have proven Eq. (\ref{eq:synch-activity}). In Methods \ref{material:iff-synchronization}, we show that this relationship is invertible in the case of linear activiation functions. In that case, if we observe synchronization of two nodes for all inputs, then we can infer that they are in a fiber. 

Now, let's prove Eq. (\ref{eq:synch-error}), which is based on the backpropagation rule--Eq. (\ref{eq:mlp-backprop}). First, note that the non-linear case of colored opfibration: $i \underset{\rm op}{\sim} j \implies i \underset{\rm fib}{\sim} j \implies \sigma'_i = \sigma'_j$. In the linear case, $\forall i: \sigma'_i = 1$, but $i, j$ are not necessarily in the same fiber. Suppose that the error is synchronized in layer $\ell+1$. Let $i, j$ be two nodes in layer $\ell$ such as $i \underset{\rm op}{\sim} j$. 

\begin{equation*}
\begin{split}
    \delta_i^{(\ell)} &= \sigma'^{(\ell)}_i \sum_k W^{(\ell+1)}_{ki}  \delta_k^{(\ell+1)}\\
    &= \sigma'^{(\ell)}_i \sum_{c \in \mathcal{C}^{\rm op}_{\ell+1}} \sum_{k \in c} W^{(\ell+1)}_{ki}  \delta_k^{(\ell+1)}\\
    &= \sigma'^{(\ell)}_i \sum_{c \in \mathcal{C}^{\rm op}_{\ell+1}} \delta_c^{(\ell+1)} \sum_{k \in c} W^{(\ell+1)}_{ki} \qquad \text{error synchronization in $\ell+1$} \\
    &= \sigma'^{(\ell)}_j \sum_{c \in \mathcal{C}^{\rm op}_{\ell+1}} \delta_c^{(\ell+1)} \sum_{k \in c} W^{(\ell+1)}_{kj}  \qquad i \underset{\rm op}{\sim} j \\
    &= \delta_j^{(\ell)}
\end{split}
\end{equation*}

Then, the error synchronization holds for layer $\ell$. Now, let's prove the base case of the induction ($\ell=N-1$).

\begin{equation*}
\begin{split}
   i \underset{\rm op}{\sim} j &\implies W^{(N)}_{ki} = W^{(N)}_{kj} \quad \& \quad \sigma'^{(N-1)}_i = \sigma'^{(N-1)}_j \\
   & \implies \sum_k W^{(N)}_{ki} \delta_k = \sum_k W^{(N)}_{kj} \delta_k \quad \& \quad \sigma'^{(N-1)}_i = \sigma'^{(N-1)}_j\\
   & \implies \delta_i^{(N-1)} = \delta_j^{(N-1)}.    
\end{split}
\end{equation*}

We have proven Eq. \ref{eq:synch-error}.

\subsection{Quasi-synchronization of activity and error in quasi-symmetries}
\label{material:quasi-synchronization}

Although synchronization theorems are not valid for quasi-symmetries, an upper bound for the variation of activity and error within quasi-fibers and quasi-opfibers can be found. If $\sigma$ is a Lipschitz function with coefficient $K_\sigma$, two nodes $i, j$ in the same quasi-fiber with threshold $\varepsilon_{\rm fib}$, we obtain:

\begin{eqnarray*}
    |h^{(\ell)}_{i} - h^{(\ell)}_{j}| &=& |\sigma \left( \sum_k W^{(\ell)}_{ik} h_k^{(\ell-1)} \right) - \sigma \left( \sum_k W^{(\ell)}_{jk} h_k^{(\ell-1)} \right)| \\
    &\leq& K_\sigma |\sum_k (W^{(\ell)}_{ik}-W^{(\ell)}_{jk}) h_k^{(\ell-1)}| \\
    &\approx& K_\sigma |\sum_{c, k\in c} (W^{(\ell)}_{ik}-W^{(\ell)}_{jk}) \hat{h}_{c}^{(\ell-1)}| \quad \text{for small $\varepsilon_{\rm fib}$ for layer $\ell-1$} \\
    &\leq& K_\sigma \sum_{c} |\hat{h}_{c}^{(\ell-1)}| |\sum_{k\in c}(W^{(\ell)}_{ik}-W^{(\ell)}_{jk})| \\
    &\leq& K_\sigma |C^{\rm fib}_{\ell-1}| \max_{c} |\hat{h}_{c}^{(\ell-1)}| \varepsilon_{\rm fib} 
\end{eqnarray*}
Here the approximation reflects the fact that in the input layer $\ell-1$ activations are approximately the mean value over the activations in the same quasi-fiber $c$: $\hat{h}_c^{(\ell-1)} \approx \frac{1}{|c|} \sum_{k \in c} h_k^{(\ell-1)}$, i.e. a quasi-fiber condition on the preceding layer.
    
For two nodes $i, j$ in the same quasi-opfiber with threshold $\varepsilon_{\rm op}$, we obtain:
\begin{eqnarray*}
    |\delta^{(\ell)}_{i} - \delta^{(\ell)}_{j}| &=& |\sigma'^{(\ell)}_i \left( \sum_k W^{(\ell)}_{ki} \delta_k^{(\ell+1)} \right) - \sigma'^{(\ell)}_j \left( \sum_k W^{(\ell)}_{kj} \delta_k^{(\ell+1)} \right)| \\
    &\approx& |\hat{\sigma}'^{(\ell)}_c| | \sum_k \left( W^{(\ell)}_{ki} -W^{(\ell)}_{kj} \right) \delta_k^{(\ell+1)}| \quad \text{for small $\varepsilon_{\rm fib}$ for layer $\ell$}  \\
    &\approx& |\hat{\sigma}'^{(\ell)}_c| | \hat{\delta}^{(\ell+1)}_{c'} | | \sum_c \sum_{k \in c}\left( W^{(\ell)}_{ki} -W^{(\ell)}_{kj} \right) | \quad \text{for small $\varepsilon_{\rm op}$ for layer $\ell+1$}  \\
    &\leq& K_\sigma |C^{\rm op}_{\ell+1}| \max_{c'} |\hat{\delta}^{(\ell+1)}_{c'}| \varepsilon_{\rm op}  \\
\end{eqnarray*}
Here the same approximation is used, replacing the error in layer $\ell+1$ with their mean value in a quasi-opfiber, $\hat{\delta}_{c}^{(\ell+1)}$, i.e. an quasi-opfiber condition for the subsequent layer. We also used the fact that $K_\sigma$ is an upper bound for $|\sigma'|$.

\subsection{Equivalence between fibration and synchronization for linear activation}
\label{material:iff-synchronization}

Here we show that synchronization of activity for all inputs implies fibration symmetry, provided node activations are linear, i.e. two nodes that are synchronized are in the same fiber. Because in the linear case, the error backpropagation is entirely analogous to the forward propagation of activity, we also conclude that error synchronization implies opfibration: two nodes with synchronized errors are in the same opfiber. Therefore, structure not only defines function, but conversely, function defines structure. This is an instance of the structure-function relation, which here appears only in the linear case where the forward pass is decoupled from the backward pass.

We will prove this by induction. First, for $\ell=1$:
\begin{equation*}
\begin{split}
\forall x: h^{(1)}_i = h^{(1)}_j &\implies \forall x: [W^{(1)}x]_i = [W^{(1)}x]_j \\
&\implies W^{(1)}_{ik} = W^{(1)}_{jk}\\
&\implies i \underset{\rm fib}{\sim} j.
\end{split}
\end{equation*}

Now, suppose that this is valid for $\ell-1$, then:
\begin{equation*}
\begin{split}
\forall x: h^{(\ell)}_i = h^{(\ell)}_j &\implies \forall x: \sum_{k}W_{ik}^{(\ell)}h^{(\ell-1)}_k = \sum_{k} W^{(\ell)}_{jk}h^{(\ell-1)}_k \\
&\implies \forall x: \sum_{c \in \mathcal{C}^{\rm fib}_\ell} \hat{h}^{(\ell-1)}_c \sum_{k \in c}W_{ik}^{(\ell)} = \sum_{c \in \mathcal{C}^{\rm fib}_\ell} \hat{h}^{(\ell-1)}_c \sum_{k \in c} W^{(\ell)}_{jk} \qquad \text{Hyp of induction} \\
&\implies \sum_{k \in c} W^{(\ell)}_{ik} = \sum_{k \in c} W^{(\ell)}_{jk} \qquad \text{because is valid for all }\hat{h}_c\\
&\implies i \underset{\rm fib}{\sim} j.
\end{split}
\end{equation*}

For opfibration symmetry and error synchronization, we have an analogous proof.

\subsection{Flavored opfibration implies error synchronization}
\label{sec:opfibration-symmetry-error-synchronization}

The backpropagation of errors in Eq. (\ref{eq:mlp-backprop}) exhibits greater complexity due to the intrinsic coupling between error gradients and forward activations mediated by the slope of the nonlinear activation $\sigma'^{(\ell)}$. Specifically, the error $\delta_i^{(\ell)}$ for node $i$ depends not only on the error signals from the next layer (mediated by the output weights $W^{(\ell)}_{ik}$) but also on the activation value $h_i^{(\ell)}$ at that node, which introduces a nonlinear coupling between $\delta_i^{(\ell)}$ and $h_i^{(\ell)}$, absent in forward propagation.

This coupling can be systematically absorbed into an output tree that governs error propagation analogously to how the input tree governs activity propagation. The key observation is that the slope --- which is a function of $h_i^{(\ell)}$ --- couples the in-balanced color assignments from the forward fiber structure (see Eq. (\ref{eq:fibration-definition})) to the connections that define the error flow in Eq. (\ref{eq:mlp-backprop}).

We formalize this by introducing \textit{connection flavors} into the output tree, where the flavors are inherited from the node colors of the forward computation. We take advantage of the concept of \textit{flavor-preserving} opfibration symmetry \cite{boldi2006graph}, which preserves flavor assignments of connections. Therefore, two nodes belong to the same flavored opfiber if:

\begin{itemize}
    \item \textbf{Flavor-preserving opfibration symmetry:}
    \begin{equation}
    i \underset{\rm op}{\sim} j \implies  \quad \forall \hat{c} \in \mathcal{C}^{\rm op}_{\ell+1}: \quad \sum_{k \in \hat{c}} W^{(\ell+1)}_{ki}= \sum_{k \in \hat{c}} W^{(\ell+1)}_{kj} \quad \& \quad q(i) = q(j) 
    \label{eq:opfibration-definition}
    \end{equation}
\end{itemize}
\noindent with flavors given by
\[ 
q(i)= \left\{
\begin{array}{ll}
      1 \qquad \qquad \qquad \text{linear activation}\\
      c(i) \in \mathcal{C}^{\rm fib}_{\ell} \qquad \text{non-linear activation}
\end{array} 
\right. 
\]

Equation (\ref{eq:opfibration-definition}) defines the opfiber partition $\mathcal{C}^{\rm op}_{\ell}$. In the linear activation case ($\sigma'$ = 1), all weights have the same flavor ($q=1$), and the forward and backward coloring decouple and can be treated independently.  In the non-linear case, the flavors must first be calculated using the fiber coloring $\mathcal{C}^{\rm fib}_{\ell}$ computed forward through the entire network as above. Then, the sum-of-weights condition of Eq. (\ref{eq:opfibration-definition}) is evaluated for each flavor separately, recursively starting at the output of the network, with all nodes in different opfibers $\mathcal{C}^{\rm op}_N = \{1,2,...,d_N\}$. 

Using this definition, we arrive at the following theorem. The proof
is given in Methods \ref{material:proof-synchronization}.

\begin{itemize}
    \item \textbf{Errors synchronization: }
     Given two nodes $i$ and $j$ in layer $\ell$, the following hold in a network with inputs $x$ and targets $y$:
    \begin{equation}
        i \underset{\rm op}{\sim} j \implies \delta^{(\ell)}_i = \delta^{(\ell)}_j \qquad \text{for any network input $x$, target $y$}\, . 
    \label{eq:synch-error}
    \end{equation} 
\end{itemize}

We note that for linear activations ($\sigma'$ = 1), the relationship between synchronization and (op)fibration becomes bidirectional ($\iff$) in both Eqs. (\ref{eq:synch-activity}) and (\ref{eq:synch-error}, as we prove in Methods \ref{material:iff-synchronization}. This means that synchronized nodes necessarily form (op)fibers, and conversely, nodes in a (op)fiber will synchronize. This represents a strict one-to-one structure-function equivalence that is difficult to demonstrate for more general non-linear dynamical systems \cite{park2013structural}. 
 
Alternative definitions for opfibration that capture the coupling with the forward pass are treated in Methods \ref{si:alternative-defs}. In the theory of gases, the linear case is analogous to non-interacting particles in an ideal gas, whereas the nonlinear case corresponds to particle interactions in a real gas, and flavor plays the role of a compression factor.

\subsubsection{Alternative definitions to capture forward-backward coupling in nonlinear Feedforward Networks}
\label{si:alternative-defs}

There are a few alternatives to our approach of using flavor-preserving opfibrations. The proof of the synchrony theorems and stability rules can be rewritten for these alternative definitions of opfibers and covers, with minor modifications. Instead of requiring flavors, one can directly require that the sum equality hold, including the slopes for all inputs $x$

\begin{itemize}
    \item \textbf{Option 1: Hybrid opfiber definition}
    \begin{equation*}
    i \underset{\rm op}{\sim} j \iff  \forall x, \forall \hat{c} \in \mathcal{C}^{\rm op}_{\ell+1}: \quad \sigma'^{(\ell)}_i \sum_{k \in \hat{c}} W^{(\ell+1)}_{ki}= \sigma'^{(\ell)}_j  \sum_{k \in \hat{c}} W^{(\ell+1)}_{kj} \, .
    \end{equation*}
\end{itemize}

These conditions are satisfied if the sum of the weight conditions is met and the slopes are synchronized between all possible inputs $x$. In the case of linear activations $\sigma'=1$, the definition simplifies to a sum constraint on the output weight alone. 

The proofs of the stability rules and the theorem actually only require error synchronization, so we can alternatively make that explicit in the definition of a cover: 

\begin{itemize}
    \item \textbf{Option 2: Hybrid cover definition}
    \begin{equation*}
    i \underset{\rm cov}{\sim} j \iff  i \underset{\rm fib}{\sim} j \quad \& \quad \forall x: \delta^{(\ell)}_i = \delta^{(\ell)}_j \, ,
    \end{equation*}
\end{itemize}
and require a covering symmetry, where we currently require a flavor-preserving opfibration. This alternative formulation of the stability rules and the cover coarsening theorem holds for all cover definitions in this work.  Indeed, note that our definition of a flavored opfibration implies a fiber in the current layer. Therefore, it automatically satisfies the cover condition \ref{eq:cover-definition}. In other words, the flavor-preserving opfibration as we have defined it happens to be a cover. In the linear case, the error synchronization condition can be relaxed to the (unflavored) opfibration condition.   

The two alternatives above swap the structural condition based on the weights, with the explicit but weaker requirement of activity/error synchronization (hence the label ``Hybrid''). Indeed, the proofs can be derived entirely from activity and error synchronization. However, we can also use a purely structural definition based solely on weights, without requiring flavored trees or synchronization, as follows.

\begin{itemize}
    \item \textbf{Option 3: Recursive cover definition}
    \begin{equation*}
    i \underset{\rm cov}{\sim} j \iff  \quad 
    i \underset{\rm fib}{\sim} j \quad \& \quad 
    \forall \hat{c} \in \mathcal{C}^{\rm cov}_{\ell+1}: \quad  \sum_{k \in \hat{c}} W^{(\ell+1)}_{ki} = \sum_{k \in \hat{c}} W^{(\ell+1)}_{kj} \, .
    \end{equation*}
\end{itemize}

This last definition leverages the knowledge that the structural requirement of fibers ensures activity synchronization and, therefore, slope synchronization $\sigma'^{(\ell)}$. It essentially couples the definition of an opfiber with the definition of the cover, which is decoupled in the linear case. 

The first two alternatives above do not have an obvious coloring algorithm based solely on weights, since the definitions depend on the activity. The last definition has the same coloring algorithm as for the flavored opfibration: One has to perform a forward coloring based on fibers, and use this to establish the cover coloring starting at the output, where in each layer, the cover is an intersection of the fibers partition with the partition due to the weight constraints.

\subsubsection{Gradient Descent Rule - Eq. (\ref{weight-update})}
\label{material:gd}

For a dataset $\mathcal{D} = \{(x,y)\}$, we denote the network's input as $h^{(0)} = x$ or $\mathbf{H}^{(0)} = x$. The output of the network is a prediction $\hat{y}$ to the desired $y$. The error of this prediction is $L = \mathcal{L}(\hat{y},y)$, where $\mathcal{L}$ is called \textit{loss function}. The weights $W^{(\ell)}$ and $\mathbf{W}^{(\ell)}$ can then be adjusted based on corrections that minimize the error $L$. Using gradient descent, the change in the weight matrix/tensor is

\begin{equation*}
    W^{(\ell)}(t) = W^{(\ell)}(t-1) - \alpha \frac{\partial L}{\partial W^{(\ell)}}(t-1)
\end{equation*}
where $\alpha$ is called the learning rate. Calculating the derivative we find:

\begin{align*}
    \text{MLP}: & \frac{\partial L}{\partial W^{(\ell)}} = \delta^{(\ell)} \cdot \left( h^{(\ell-1)} \right)^T, \\
    \text{CNN}: & \frac{\partial L}{\partial \mathbf{W}^{(\ell)}} = \delta^{(\ell)}  \ast \mathbf{H}^{(\ell-1)}
\end{align*}   
where $\delta^{\ell}$ is the error signal for the nodes in layer $\ell$.

\subsubsection{Proof of Synchronized Learning - Eq. (\ref{eq:synch-learning})} 
\label{material:proof-synchron}

Consider two nodes $i$ and $j$ in the layer $\ell-1$ in the same fiber and two nodes $m$ and $n$ in the layer $\ell$ in the same opfiber (see red and green nodes in Fig. \ref{fig:theoretical_results}a). Using Eqs. (\ref{eq:synch-activity})-(\ref{eq:synch-error}), we obtain $h^{(\ell-1)}_i = h^{(\ell-1)}_j$ and $\delta^{(\ell)}_m = \delta^{(\ell)}_n$. 

Then,
\begin{equation*}
\begin{split}
\Delta W^{(\ell)}_{mi} = -\alpha \frac{\partial L}{\partial W_{mi}^{(\ell)}} &= -\alpha \delta^{(\ell)}_m h^{(\ell-1)}_i \\
&= -\alpha \delta^{(\ell)}_m h^{(\ell-1)}_j = \Delta W^{(\ell)}_{mj}\\
&= -\alpha \delta^{(\ell)}_n h^{(\ell-1)}_i = \Delta W^{(\ell)}_{ni}\\
&= -\alpha \delta^{(\ell)}_n h^{(\ell-1)}_j = \Delta W^{(\ell)}_{nj}   
\end{split}
\end{equation*}

\subsection{Covering symmetries emerge during stochastic gradient descent}
\label{sec:cover-theorem}

The final symmetry that we introduce in DNNs is the covering symmetry. Two nodes belong to the same cover iff they are in the same fiber and
the same flavored opfiber:

\begin{itemize}
    \item \textbf{Covering symmetry:}
    \begin{equation}
    \label{eq:cover-definition}
    i \underset{\rm cov}{\sim} j \iff i \underset{\rm fib}{\sim} j \quad \& \quad    i \underset{\rm op}{\sim} j   \, .
    \end{equation}
\end{itemize}

This condition defines a covering partition of the network
$\mathcal{C}^{\rm cov}_{\ell}$, which is preserved under GD. Specifically, as a consequence of \textit{synchronized learning}, once two nodes are in a single covering at learning time step $t$, the weight update with the gradient descent algorithm keeps the nodes in the covering at time $t+1$. This is formalized in the {\em Cover Coarse-Graining Theorem} \ref{theorem:covering} (proof in Methods \ref{material:covering-theorem-proof})
. 
We proof this theoretical result in Methods \ref{material:covering-theorem-proof}. The theorem implies that once multiple nodes are in the same cover, they cannot exit that cover, and thus constitute an invariant set. However, distinct covers can merge to form larger (coarser) covers. The proof of this Cover Coarse-Graining Theorem is based on stability rules (see Methods \ref{material:covering-theorem-proof}). These provide an intuition for how fibration and opfibration symmetries develop during learning. The stability rules state that the fibers in the layer $\ell$ preserve the fibers in the layer $\ell+1$, while the opfibers preserve the opfibers in the layer $\ell-1$, as shown in Fig.~\ref{fig:theoretical_results}b. Thus, opfiber formation tends to propagate backward from the output layer, while fiber formation tends to propagate forward from the input as learning progresses. Note that for the non-linear case, the opfibers have a flavoring of the connections to capture their dependence on the forward pass. This introduces an asymmetry in the stability rules, reflected in Fig~\ref{fig:theoretical_results}b, with connection colors defined at the output, but not at the input.

\subsection{Proof of Theorem \ref{theorem:covering}: Cover Coarse-Graining Theorem}
\label{material:covering-theorem-proof}

First we will proof stability rules for fibers and opfibers, which are a direct consequence of synchronized learning under gradient descent. Then we use these rules for a recursive proof of the Cover Coarse-Graining Theorem. Then we briefly discuss stability of covers under stochastic gradient descent. 

\subsubsection{Opfiber, Fiber and Cover Stability Rules} 

Here we will proof three stability rules used to proof the cover coarsening theorem. The opfiber stability rule states that nodes in the same cover (blue nodes in Fig. \ref{fig:theoretical_results}b) at learning time $t$ will belong to the same opfiber after one GD learning step $t+1$, if their output nodes are on the same opfiber. The fiber stability rule states that nodes in the same cover at time $t$ will belong to the same fiber the next time $t+1$, if their input nodes are on the same fiber. Both rules are combined in the cover stability rule, which states that nodes in the same covering at time $t$ remain in the same covering in the subsequent time $t+1$.

During the training process, the network weights $W$, the symmetries, and partitions of the nodes $\mathcal{C}$ will depend on the time $t$.

During gradient descent training of an FNN, \textit{learning
  synchronization} ensures that:
\begin{enumerate}
    \item Nodes in the same cover $i
      \underset{cov,\ell,t}{\sim} j$ will belong to the same opfiber at the next learning times step $i\underset{op,\ell,t+1}{\sim}j$ if
      $\mathcal{C}^{\rm op}_{\ell+1}(t+1)$ is coarser than
      $\mathcal{C}^{\rm op}_{\ell+1}(t)$.
    \item Nodes in the same cover $i
      \underset{cov,\ell,t}{\sim}j$ will belong to the same fiber at the
      next learning time step $i\underset{fib,\ell,t+1}{\sim}j$ if
      $\mathcal{C}^{\rm fib}_{\ell-1}(t+1)$ is coarser than
      $\mathcal{C}^{\rm fib}_{\ell-1}(t)$.
    \item Nodes in the same cover $i
      \underset{cov,\ell, t}{\sim} j$ will belong to the same cover at the next learning time step $i\underset{cov,\ell,t+1}{\sim}j$ if $\mathcal{C}^{\rm fib}_{\ell-1}(t+1)$ is coarser than  $\mathcal{C}^{\rm fib}_{\ell-1}(t)$ and $\mathcal{C}^{\rm op}_{\ell+1}(t+1)$ is coarser than $\mathcal{C}^{\rm op}_{\ell+1}(t)$.   
\end{enumerate}

\textbf{Proof} (1) Let $i \underset{cov,\ell,t}{\sim} j$ be nodes in
layer $\ell$. For learning time step $t+1$, for $\forall \hat{c} \in
\mathcal{C}^{\rm op}_{\ell+1}(t+1)$:

\begin{equation*}
\begin{split}
    \sum_{k \in \hat{c}} W^{(\ell+1)}_{ki}(t+1) &= \sum_{k \in \hat{c}} W^{(\ell+1)}_{ki}(t) + \Delta W^{(\ell+1)}_{ki}(t) \\
    &= \sum_{k \in \hat{c}} W^{(\ell+1)}_{ki}(t) + \Delta W^{(\ell+1)}_{kj}(t).
\end{split}
\end{equation*}

The last equality is valid due to $i \underset{fib,\ell,t}{\sim} j$ and activity synchronization. Now, note that

\begin{equation*}
\begin{split}
    \sum_{k \in \hat{c}} W^{(\ell+1)}_{ki}(t) &= \sum_{r \in \mathcal{C}^{\rm op}_{\ell+1}(t)} \sum_{k \in \hat{c} \cap r} W^{(\ell+1)}_{ki}(t) \\
    & = \sum_{r \in \mathcal{C}^{\rm op}_{\ell+1}(t)} \Theta(r \subset \hat{c}) \sum_{k \in r} W^{(\ell+1)}_{ki}(t) \qquad \textit{Hyp) $\mathcal{C}^{\rm op}_{\ell+1}(t+1)$ is coarser than $\mathcal{C}^{\rm op}_{\ell+1}(t)$} \\
    & = \sum_{r \in \mathcal{C}^{\rm op}_{\ell+1}(t)} \Theta(r \subset \hat{c}) \sum_{k \in r} W^{(\ell+1)}_{kj}(t) \qquad \textit{Hyp) $i \underset{op,\ell,t}{\sim} j$}\\
    & = \sum_{k \in \hat{c}} W^{(\ell+1)}_{kj}(t).        
\end{split}
\end{equation*}

Then, we obtain $\sum_{k \in \hat{c}} W^{(\ell+1)}_{ki}(t+1) = \sum_{k \in \hat{c}} W^{(\ell+1)}_{kj}(t+1)$. That means $i \underset{op,\ell,t+1}{\sim} j$.

The proof of (2) is similar. 
The proof of (3) is the consequence of (1) and (2).

\subsubsection{Recursive proof of cover coarsening theorem}
The cover coarsening theorem indicates nodes in the same cover $i \underset{cov,\ell, t}{\sim} j$ will belong to the same cover at the next time $i\underset{cov,\ell,t+1}{\sim}j$. The theorem relaxes the conditions of the cover stability rule.

For $\ell=0$ and $\ell=N$, $\mathcal{C}^{\rm fib}_0(t) = \{[1],[2],...,[d_0]\}$, and $\mathcal{C}^{\rm op}_N(t) = \{[1],[2],...,[d_N]\}$ that remain constant over time $t$. 

We apply fiber stability rule for layer $\ell=1$. Nodes in the same cover $i \underset{cov,\ell=1,t}{\sim}j$ will belong to the same fiber at the next time $i\underset{fib,\ell=1,t+1}{\sim}j$ because $\mathcal{C}^{\rm fib}_{0}(t+1)$ is coarser than $\mathcal{C}^{\rm fib}_{0}(t)$. This means $\mathcal{C}^{\rm fib}_{1}(t+1)$ is coarser than $\mathcal{C}^{\rm fib}_{1}(t)$.

With the same argument, we apply the fiber stability rule for layer $\ell=2$, and we obtain $\mathcal{C}^{\rm fib}_{2}(t+1)$ is coarser than $\mathcal{C}^{\rm fib}_{2}(t)$. 

By iterating over all layers, we obtain $\forall \ell: i \underset{cov,\ell, t}{\sim} j \implies i \underset{fib,\ell, t+1}{\sim} j$.

Applying the same reasoning with the opfiber stability rule from layer $\ell=N$, then $\forall \ell: i \underset{cov,\ell, t}{\sim} j \implies i \underset{op,\ell, t+1}{\sim} j$.

Using both results, we obtain: $\forall \ell: i \underset{cov,\ell, t}{\sim} j \implies i \underset{cov,\ell, t+1}{\sim} j$.

The nodes in the same cover at learning time step $t$,  will belong to the same cover in the next time step. Equivalently, 
$\mathcal{C}^{\rm cov}_{\ell}(t+1)$ is coarser than  $\mathcal{C}^{\rm cov}_{\ell}(t)$.  

\subsubsection{Stability of covers}

Theorem \ref{theorem:covering} can be interpreted as follows: \textit{If network parameters $\theta_t$ at time $t$ satisfy $i \underset{\rm cov}{\sim}j$; then $i \underset{\rm cov}{\sim} j$ still holds for  $\theta_{t+1}$}. That is, the parameter set satisfying $i \underset{\rm cov}{\sim} j$ is invariant under the weight update in GD. 

Chen \textit{et al.} \cite{chen2023stochastic} have shown that any invariant parameter set of GD is a stable attractor in the SGD algorithm, provided that there is a sufficiently large learning constant. The intuition is that stochastic fluctuations transverse to a cover are damped, drawing the parameters toward the symmetric set, while fluctuations along it vanish due to synchronized learning. As a result, the cover theorem implies that the covering symmetries are stable attractors of SGD. 

\subsection{Fibers coloring for sums, convolutions, gates and concatenations}
\label{material:hypergraphs}

We have discussed fibers for the canonical MLP architecture. here we explain how they apply to the most common architectural motives found in modern deep networks. Many of these, such as sums and gating, can be understood as hypergraphs, and therefore we are defining fibration rules here for hypergraphs. 

\begin{itemize}
    \item {\bf Sums:}
    The most ubiquitous motive are nodes that are added together. They can be a simple as the sum of a bias term, ${\bf W} {\bf h} + {\bf b}$ or the sum of two sets of inputs with their corresponding weights, ${\bf W}_h {\bf h} + {\bf W}_x {\bf x}$, or residual connections, ${\bf x}+{\bf W}_h {\bf h}$. They can be simply treated by concatenating weights and inputs. For example: ${\bf W}_h {\bf h} + {\bf W}_x {\bf x} = [{\bf W}_h | {\bf W}_x] [{\bf h}^T | {\bf x}^T]^T = \widehat{\bf W} \widehat{\bf h}$. The fiber equations are then simply applied to $\widehat{\bf W}$, and the nodes in inputs ${\bf x}$ and ${\bf h}$ contribute their own colors. In the case of residual connections, an identity matrix is used when concatenating weights. The orthogonality of this identity makes the formation of fibers unlikely unless they were already in the input ${\bf x}$. In other words, residual connections will tend to break any fibers that would otherwise have formed in a layer. 
    
    \item {\bf Convolutions:} Most prominent in CNNs, convolutions, $\mathbf{W} * \mathbf{H}$, where a matrix of "kernels" of filters $\mathbf{W}$ mixes feature channels $\mathbf{H}$, while each kernel mixes time samples, pixels, or tokens, depending on the application.  We can simply ignore the time/pixel/token dimensions and treat each filter as a scalar weights. For instance, the fibration condition is simply:
    
    \begin{equation*}
    \sum_{k \in c}  \hat{\mathbf{W}}[:,:,i,k] =  \sum_{k \in c} \mathbf{W}^{(\ell)}[:,:,j,k]. 
    \end{equation*}

    See example in Fig.~\ref{fig:compression_architec}a. The usual pooling operates only along the time/pixel/toke dimensions and does not affect colors only given to feature channels.  

    \item {\bf Concatenations:}
    When two vectors of nodes are concatenated, then each contributes their own color, similar to the way sums where handled above. 
        
    \item {\bf Gates:}
    In RNNs such as the LSTM it is common to use gates of the form, ${\bf f} \odot {\bf x}$, where in this case an input vector ${\bf x}$ is multiplied element-wise with a forget gate ${\bf f}$. This represents a multiplicative two-node interaction. For a fiber to emerge in this product, two nodes must be a fiber in the corresponding nodes of ${\bf x}$ \textit{and} in a fiber of nodes of ${\bf f}$, as exemplified with the colors in Fig.~\ref{fig:compression_architec}b. Two nodes in the product have the same colors, if they had the same colors in both factors --- this color combination rule is the same \textit{and} operation as we used for covers, but here it applies to the factors of the product. A similar element-wise product occurs in transformers when the key and query vectors are multiplied, ${\bf k} \odot {\bf q}$. The coloring can be treated the same as for gates, but the two structures differ for the fibration compression rule as shown in Fig.~\ref{fig:compression_architec}c and explained in Methods~\ref{sec:fiber-compression-hypergraph}. As with convolutional networks, the toke dimensions are not colored, and thus the attention matrix itself, square in the number of tokens, does not affect colors.    
\end{itemize}

Identifying symmetries in the time/pixel/token dimensions will be left to future work. 

\begin{figure}
    \centering
    \includegraphics[width=0.8\linewidth]{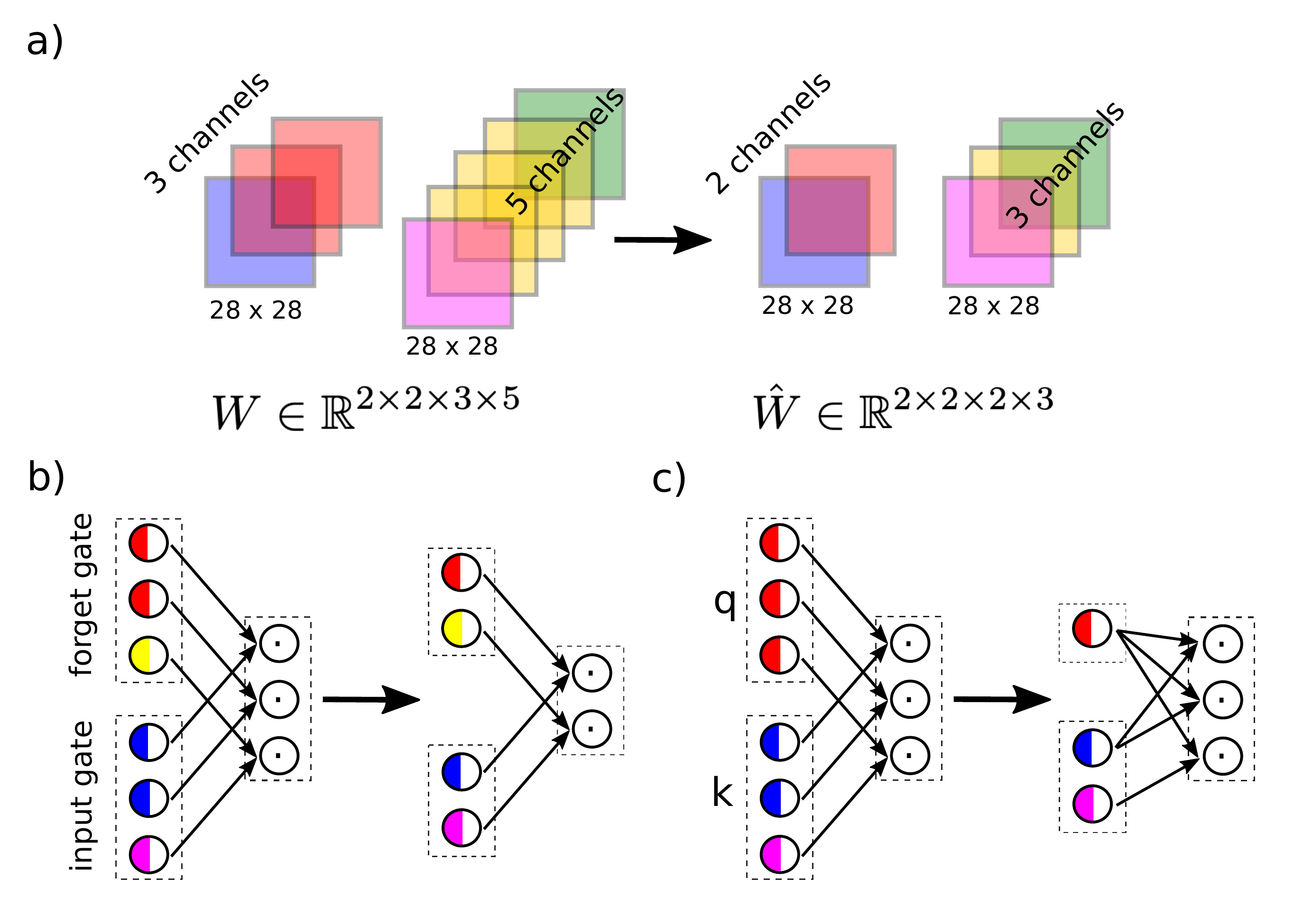}
    \caption{\textbf{(a)} Example of fibration compression of a convolution layer. It reduces the number of channels (3 to 2 and 5 to 3).  \textbf{(b)} Example of fibration compression of gates in LSTM. The base has an identical number of nodes across all gates (here 2 nodes per gate). \textbf{(c)} The fibration compression for attention modules   preserves the count of interactions (3 multiplications) even when the number of fibers of the $k$ and $q$ differ.}
    \label{fig:compression_architec}
\end{figure}

\subsection{Fibration compression rules}
\label{material:fiber-compression}

Now we want to derive the fibration compression rule for the parameters of a base network $B_{\rm fib}$, so that the activity of the nodes remains unchanged, i.e., the base network maintains forward computation of the original network $G$. We will do this first for the MLP, and then discuss compression for other architectures. 

\subsubsection{Fibration compression for a feedforward network}
\label{sec:fiber-compression-MLP}

Let $\hat{h}_{c'}$ be the activity of a fiber $c' \in \mathcal{C}^{\rm fib}_\ell$, then it follows that:
\begin{equation}
\begin{split}
\hat{h}^{(\ell)}_{c'} & = h^{(\ell)}_i \qquad \forall i \in c' \qquad \text{(activity synchronization)}\\
 & = \frac{1}{|c'|} \sum_{i \in c'} h^{(\ell)}_i \qquad \text{only $\approx$  for $\varepsilon_{\rm fib}>0$}\\
 & = \frac{1}{|c'|} \sum_{i \in c'} \left[ W^{(\ell)}_{ik} h^{(\ell)}_k + b^{(\ell)}_{i} \right]\\
 & = \frac{1}{|c'|} \sum_{i \in c'} \sum_{c \in \mathcal{C}^{\rm fib}_{\ell-1}} \sum_{k \in c} W^{(\ell)}_{ik} h^{(\ell)}_k + \frac{1}{|c'|} \sum_{i \in c'} b^{(\ell)}_{i}\\ 
 & = \sum_{c \in \mathcal{C}^{\rm fib}_{\ell-1}} \left[ \frac{1}{|c'|} \sum_{i \in c'}  \sum_{k \in c} W^{(\ell)}_{ik} \right]  \hat{h}^{(\ell)}_c + \left[ \frac{1}{|c'|}  \sum_{i \in c'} b^{(\ell)}_{i} \right]\\ 
 & = \sum_{c \in \mathcal{C}^{\rm fib}_{\ell-1}} \hat{W}_{c'c} \hat{h}^{(\ell)}_c + \hat{b}^{(\ell)}_{c} \, ,\\ 
\end{split}
\end{equation}
where$\hat{W}_{c'c}^{(\ell)} = \frac{1}{|c'|} \sum_{i \in c'}  \sum_{k \in c} W^{(\ell)}_{ik}$ and $\hat{b}^{(\ell)}_{c'} = \frac{1}{|c'|} \sum_{i \in c'} b^{(\ell)}_{i}$. These matrices $\hat{W}_{c'c} $ and vectors $\hat{b}^{(\ell)}_{c'}$ are the parameters of the base $B_{\rm fib}$. They are obtained by adding the weights with the same input colors and averaging the weights with the same output colors (Fig. \ref{fig:theoretical_results}c)

A similar equation can be derived for a base network $B_{\rm op}$ whose nodes are the opfibers of the original network $G$ and the backward computation of the error is preserved. 

Then compression to a covering base can be chosen to preserve the input or output trees. This is represented pictorially for the same covering base in a network with binary weights in Fig~\ref{fig:cover_symm_compression_two_ways}. Compressing a graph into its covering base where both forward and backward computations are preserved is only possible when all covers have the same cardinality $|c'|$. This is too strong of a requirement and has not been explored in more detail here. 

\subsubsection{Fibration compression for a convolutional layer in an CNN}
\label{sec:fiber-compression-CNN}

A similar line of reasoning works for CNNs where the nodes of the computational graph are the channels in the convolutions. The definition of fibers and lifting for MLPs (Eqs. \ref{eq:fibration-definition}, \ref{eq:fib-weight-compression}) can be generalized to CNNs. These equations are applied to the channels $i,k$ in a layer $\ell$; where $W^{(\ell)}_{ik}$ is a $D^2$-dimensional vector that represents the kernel between the channels $i$ and $j$ (see Fig. \ref{fig:theoretical_results}a).

More precisely,
\begin{equation}
\begin{split}
\hat{\mathbf{W}}[:,:,c,c'] &= \frac{1}{|c'|} \sum_{i \in c'}  \sum_{k \in c} \mathbf{W}^{(\ell)}[:,:,k,i], \\
\hat{\mathbf{b}}^{(\ell)}_{c'} &= \sum_{i \in c'} \mathbf{b}^{(\ell)}_{i}.
\end{split}
\end{equation}

\subsubsection{Fibration compression for an Hypergraph}
\label{sec:fiber-compression-hypergraph}

Let us generalize the concepts of fibrations and lifting to any computational hypergraphs, where two or more nodes interact before acting on the subsequent node. The most common interaction is a multiplicative modulation or ``gating” of activity. 

Suppose that there are two triplets of nodes $(p,i,j)$ and $(p',i',j')$ such as $h_p = h_i . h_n$ and $h_{p'} = h_{i'} . h_{n'}$ where $h$ is the activity of the node. We will say $p \sim p'$ if and only if $i\sim i'$ and $j \sim j'$. Note $p \sim p' \implies h_p = h_{p'}$.

Many times, this modulation appears between node layers $z = x \odot y$, where $x,y,z \in \mathbb{R}^d$ (product between forget and input gate in LSTM or product between key and query in the attention mechanism). The activity in $z$ is $z_i = x_i y_i = \hat{x}_c \hat{y}_{c'}$ when $i$ is in the fiber $c$ of the layer $x$ and $i$ is in the fiber $c'$ of the layer $y$.

\begin{itemize}
    \item \textbf{LSTM}. An important property of fibration in the gates in an LSTM is that they have the same fiber distribution (see the forget gate and the input gate in Fig. \ref{fig:theoretical_results}b). This implies that all layers ($x,y,z \in \mathbb{R}^{d}$) will be compressed to the same number ($\hat{x},\hat{y},\hat{z} \in \mathbb{R}^{d'}$) and the dynamics is $\hat{z}_c = \hat{x}_c\hat{y}_c$ where $c$ is a fiber. Therefore, if nodes are in the same fiber for the forget gate (see red nodes in Fig.~\ref{fig:theoretical_results}b), their corresponding nodes in the input gate (see blue nodes), output gate, cell input and hidden state, will also be in the same fiber.  Consequently, the base has an identical number of nodes across all gates (2 nodes per gate in Fig.~\ref{fig:theoretical_results}b).
    \item \textbf{Attention Mechanism}. Within an attention module of Transformers, the number of fibers may vary for the query (q), key (k), and value (V) projections (see Fig.~\ref{fig:theoretical_results}c). Then we can reduce the number of nodes in $q$ and $k$ from $d$ to $d'_q$ and $d'_k$, respectively; but we cannot reduce the number of nodes in $z = q \odot k$, i.e. fibration compression preserves the number of interactions (see Fig. \ref{fig:theoretical_results}c).
\end{itemize}

\subsection{Datasets and tasks used in empirical tests}
\label{sec:experiments_cfg}

Table~\ref{tab:training-configs} summarizes the architecture and training hyperparameters for  experimental setups: MNIST and ImageNet classification. Ten independent training runs were conducted for each experiment. The average training accuracy over epochs is shown in Figs. \ref{fig:results-formation-and-pruning}a and c.

\begin{table}[htbp]
\centering
\caption{Training configurations for MLP-MNIST and CNN-ImageNet experiments.}
\label{tab:training-configs}
\begin{tabular}{lll}
\toprule
\textbf{Parameter} & \textbf{MLP (MNIST)} & \textbf{CNN (ImageNet)} \\
\midrule
\multicolumn{3}{l}{\textit{Architecture}} \\
\quad Input & 28$\times$28 Binary & 32$\times$32 RGB \\
\quad Hidden layers & 3 $\times$ 500 nodes (dense) & Conv(32, 5$\times$5), Conv(64, 3$\times$3), \\
                   &                              & Dense(128), Dense(128) \\
\quad Output & 10 nodes & 2 nodes \\
\midrule
\multicolumn{3}{l}{\textit{Dataset}} \\
\quad Classes & 10 digits (0--9) & 1{,}000 classes \\
\quad Training set & 60K images & 600K images \\
\quad Test set & 10Kimages & 100K images \\
\midrule
\multicolumn{3}{l}{\textit{Training}} \\
\quad Loss & Cross Entropy & Cross Entropy \\
\quad Optimizer & Adam & SGD \\
\quad Learning rate & 0.001 & 0.001 \\
\quad Momentum & -- & 0.9 \\
\quad Batch size & 100 & 100 \\
\quad Epochs & 600 & 300 \\
\quad Independent runs & 10 & 10 \\
\bottomrule
\end{tabular}
\end{table}

Table~\ref{table:RL} lists the hyperparameters used to train the Proximal Policy Optimization (PPO) agent: a deep neural network consisting of a sequential stack of three Conv2d layers, a linear layer, and a LSTM.
The agent runs across 65 parallel environments to maximize sample diversity,
accumulating experience in batches of 30{,}720 transitions before each
gradient update. A short BPTT horizon of 2 is adopted to limit the
computational cost of backpropagating through the LSTM while still allowing
the recurrent state to capture short-term temporal dependencies. 

Symmetry formation during LSTM training with Proximal Policy Optimization (PPO) is not monotonic. We attribute this to the entropy term in the PPO, which encourages random exploration and prevents premature collapse of the policy toward deterministic actions \cite{schulman2017proximal}. We hypothesize that the entropy term acts as a symmetry-regulating mechanism whose effect depends on the concentration of the policy. Early in training, when the policy is nearly uniform, the curvature of the entropy term is positive, reinforcing existing quasi-covers and promoting redundancy across nodes. As the policy focuses on a smaller set of actions, this curvature changes sign, destabilizing symmetric configurations, and allowing nodes to specialize. This dual role provides a mechanistic account of the non-monotonic evolution of quasi-covers observed in the black curve in Fig. \ref{fig:results-formation-and-pruning}e, where the proportion of quasi-covers first decreases, then gradually increases as specialization breaks them. 

\begin{table}[htbp]
\centering
\caption{Training hyperparameters for the PPO agent on Atari Beam Rider.}
\label{table:RL}
\begin{tabular}{ll}
\toprule
\textbf{Hyperparameter} & \textbf{Value} \\
\midrule
\multicolumn{2}{l}{\textit{Environment}} \\
\quad Environment & Atari Beam Rider \\
\quad Number of environments & 65 \\
\midrule
\multicolumn{2}{l}{\textit{Training}} \\
\quad Total agent timesteps & 9{,}216K \\
\quad Batch size & 30{,}720 \\
\quad Minibatch size & 1{,}024 \\
\quad Learning rate & $4\times10^{-4}$ \\
\quad BPTT horizon & 2 \\
\bottomrule
\end{tabular}
\end{table}

Table~\ref{table:transformer} details the architecture and training
configuration of the sequence-to-sequence transformer. The model uses a
symmetric encoder--decoder stack of three layers each, with eight attention
heads per layer. The asymmetry between source (10{,}837 tokens) and target
(19{,}214 tokens) vocabulary sizes reflects the morphological richness of
German relative to English. Training is regularized via early stopping with
a patience of 20 epochs. 

\begin{table}[htbp]
\centering
\caption{Architecture and training hyperparameters for the sequence-to-sequence
transformer on the Multi30k German-to-English translation task.}
\label{table:transformer}
\begin{tabular}{ll}
\toprule
\textbf{Parameter} & \textbf{Value} \\
\midrule
\multicolumn{2}{l}{\textit{Architecture}} \\
\quad Number of encoder layers & 3 \\
\quad Number of decoder layers & 3 \\
\quad Number of attention heads & 8 \\
\quad Source vocabulary size & 10{,}837 \\
\quad Target vocabulary size & 19{,}214 \\
\midrule
\multicolumn{2}{l}{\textit{Training}} \\
\quad Batch size & 128 \\
\quad Learning rate & $1\times10^{-4}$ \\
\quad Adam $\beta_1$ & 0.9 \\
\quad Adam $\beta_2$ & 0.98 \\
\quad Early stopping patience & 20 epochs \\
\bottomrule
\end{tabular}
\end{table}

\subsection{Refinement algorithm for balanced colorings in feedforward networks}
\label{material:refinement-coloring}

A {\it coloring} $c$ of a graph $G$ is a map $c: N_G \rightarrow
\mathcal{C}$ where $c(u)$ is called the color of the node $u$ and
$\mathcal{C}$ is called the set of \textit{colors}. The coloring $c$
is \textit{in-balanced} if $c(u) = c(v)$ implies that $T_u$ and $T_v$
are color isomorphic. The minimal in-balanced coloring is a
in-balanced coloring of a graph with the minimal number of colors. In
terms of synchronization, nodes inside the same subset of the
in-balanced coloring partition (i.e. nodes with the same color) can
synchronize its activities, since they receive the same color inputs
from the same synchronized nodes.

A refinement algorithm of $G$ starts with an initial coloring $c_0$ and produces a new coloring $c_t$ in each iteration $t$, based on some criteria, until a desired property is achieved (e.g. in-balanced coloring).

In a graph without weighted connections and without labeled nodes, the method can be summarized as follows.

\begin{enumerate}
\item $c_0$ assigns a trivial color to each vertex $v$ (e.g. $c_0(v)$=1). 
\item $c_{i+1}(v)=\left(c_{i}(v),\{\{c_{i}(w)\mid w{\text{ is a neighbor of }}v\}\}\right)$
\end{enumerate}

At some point, it stabilizes $c_{t+1}(u)=c_{t}(u)$ $\forall u \in N_G$, $t>T$. 
We designed a new method for a layered feedforward network. First, let us reformulate Eq. (\ref{eq:fibration-definition})

\begin{equation}
\begin{split}
    i \underset{\rm fib}{\sim} j \iff \forall r: \sum_{k| c(k)=r} W^{(\ell)}_{jk} = \sum_{k| c(k)=r} W^{(\ell)}_{ik} &\iff \forall r: \hat{W}^{(\ell)}_{ir} = \hat{W}^{(\ell)}_{jr}  \\
    &\iff \hat{W}^{(\ell)}_{i,:} = \hat{W}^{(\ell)}_{j,:} \\
    &\iff d(\hat{W}^{(\ell)}_{i,:}, \hat{W}^{(\ell)}_{j,:}) = 0 \\
    &\iff 1-\hat{W}^{(\ell)}_{i,:} . \hat{W}^{(\ell)}_{j,:} = 0 \qquad \text{if $\hat{W}$ is L2-normalized}
\end{split}
\label{eq:distance-fibration}
\end{equation}
where $\hat{W}^{(\ell)}_{ir} = \sum_{k| c(k)=r} W^{(\ell)}_{ik}$. 

Given that these conditions are impossible to satisfy exactly for continuous weights, we transform Eq. (\ref{eq:distance-fibration}) into a set of inequality constraints parameterized by a tolerance $\varepsilon$. More precisely, we define the distance matrix $D^{(\ell)}_W = 1 - \hat{W}^{(\ell)} \hat{W}^{(\ell)T}$ and we use it to calculate the colors/clusters ($i \underset{\rm fib}{\sim} j$) via agglomerative clustering with distance threshold $\varepsilon$: 

\begin{enumerate}
\item $c^{(\ell=0)} \in \mathcal{C}^{\rm fib}_{0}$ assigns unique colors to each of the input features ($\ell=0$).
\item For $\ell=1,...,N$:

\begin{enumerate}[label*=\arabic*.]
    \item Calculation of $\hat{W}^{(\ell)}_{ir} = \sum_{k| c(k)=r} W^{(\ell)}_{ik}$ $\forall r \in  \mathcal{C}^{(\ell)}$
    \item Normalization of $\hat{W}^{(\ell)}$.
    \item Calculation of the matrix $D^{(\ell)}_W = I - \hat{W}^{(\ell)} \hat{W}^{(\ell)T}$ 
    \item Agglomerative Clustering of the nodes in $\ell$ based on the distance matrix $D^{(\ell)}_W$ (see Alg. \ref{alg:agglomerative}). We use a distance threshold $\varepsilon \in (0,2)$.
    \item Fibers $c^{(\ell)} \in \mathcal{C}^{\rm fib}_{\ell}$ are defined as the clusters. 
\end{enumerate}
\end{enumerate}

\begin{algorithm}
\caption{Agglomerative Clustering}\label{alg:agglomerative}
\begin{algorithmic}[1]

\State Initialize clusters $\mathcal{C}={{1},{2},...,{n}}$ 

\While{$|\mathcal{C}| > 1$}
    \State $\forall c,c' \in \mathcal{C}: d_{\rm complete}(c, c') = \underset{i \in c, j \in c'}{\max} D_W(i,j)$
    \State $d_{min} = \underset{c,c'\in \mathcal{C}}{\min} d_{\rm complete}(c, c')$
    \State If $d_{min}>\varepsilon$: stop
    \State $\mathcal{C} \gets$ Merge the clusters whose distance is less than $d_{min}$.
\EndWhile
\State \Return $\mathcal{C}$

\end{algorithmic}
\end{algorithm}

Note that agglomerative clustering employs a distance threshold $\varepsilon$ to determine cluster merging (see Alg. \ref{alg:agglomerative}). According to Eq. (\ref{eq:distance-fibration}), this is equivalent to treating the equalities $\sum_{k \in c} W^{(\ell)}_{ik} = \sum_{k \in c} W^{(\ell)}_{jk}$ with a tolerance parameter $\varepsilon$, which defines the tolerance to a weak fibration breaking (see Step 2.4). In other words, all approximate fibration symmetries should be regarded as quasi-fibrations \cite{boldi2022quasifibrations} or quasi-balanced colorings \cite{leifer2022symmetry}.

The coloring generated in the feedforward pass generates $\mathcal{C}^{\rm fib}_{\ell}$ for all layers. The same algorithm run backwards starting at the output generates $\mathcal{C}^{\rm op}_{\ell}$. The number of operations to execute this coloring algorithm in a layered feed-forward graph scales linearly with the parameter size in each layer and linearly with the number of layers $N$. 

For recurrent networks, the fibration and opfibration definitions apply similarly, except that partitioning $\mathcal{C}^{\rm fib}_{\ell}$ and $\mathcal{C}^{\rm op}_{\ell}$ must be iterated recursively forward and backward, respectively, until convergence. Convergence requires, at most, the times of the longest loop in the recurrence \cite{norris1995universal}. The code is available at \url{github.com/MakseLab/fibrations_in_dnns}.

\subsection{Optimal quasi-cover thresholds for a given loss threshold} 
\label{si:loss-threshold-search}

One can use the same $\varepsilon$ for all layers, or select layer-specific quasi-fiber and quasi-opfiber thresholds. To find the optimal $\varepsilon$ efficiently, it is helpful to have an efficient estimate of this $\Delta L$. We estimate the loss using gradients $G_{ik} = \frac{\partial L}{\partial W_{ik}}$ that have already been computed during training. The weight update $\Delta {\bf W}$ by compression changes the loss by approximately:
\begin{equation*}
\Delta L \approx \sum_{i,k} \frac{\partial L}{\partial W_{ik}} \Delta W_{ik} = \text{tr}({\bf G}^{T} \Delta {\bf W}).
\end{equation*}
Note that $\Delta W$ depends on $\varepsilon$, while $G$ does not, since it is computed before compression. The compression of the entire network, each layer $\ell$ contributes a term $\Delta L^{(\ell)}$; i.e.,
\begin{equation}
    \Delta L(\vec{\varepsilon}) = \sum_{l} \Delta L^{(\ell)}(\vec{\varepsilon}) ,
    \label{eq:total-loss-change}
\end{equation}
where $\vec{\varepsilon} = (\varepsilon^{(1)}_{\rm fib},\varepsilon^{(2)}_{\rm fib},...,\varepsilon^{(N)}_{\rm fib}, \varepsilon^{(1)}_{\rm op},\varepsilon^{(2)}_{\rm op},...,\varepsilon^{(N)}_{\rm op})$. On the other hand, for each $\vec{\varepsilon}$, the collapsed network has $P(\vec{\varepsilon})$ parameters. $P$, as a function of each dimension of $\vec{\varepsilon}$, is decreasing. We define a threshold for the variation of the loss function $\varepsilon_{Loss}$ and want to solve the following constrained optimization problem:
\begin{equation}
    \min_{\vec{\varepsilon}} P(\vec{\varepsilon}) \quad \text{with } \quad \Delta L(\vec{\varepsilon}) \leq \varepsilon_{\rm Loss}. 
    \label{eq:optimization_equation}
\end{equation}

With this formulation, we seek the maximum possible compression, provided that the collapsed network performs similarly to the original network (in terms of loss function). $\varepsilon_{\text{Loss}}$ must be defined a priori as a fixed value. Unlike $\varepsilon$, which we limit to $\le 1$ in the search by normalizing the weight matrix in the coloring algorithm (see Eq. (\ref{eq:distance-fibration}), $\varepsilon_{\text{Loss}}$ is not confined to a bounded domain. A natural question is how to determine a reasonable value for $\varepsilon_{\text{Loss}}$. Each network is designed to perform a specific task, whose performance is measured by a metric $A$ (e.g., accuracy). Although there is no exact functional relationship between $A$ and $L$, it is generally expected that $A$ increases as $L$ decreases. Moreover, beyond a certain point during training, $A$ stops increasing and stabilizes, while $L$ continues to decrease (see Fig.~\ref{fig:metrics_MLP_CNN}). The value of $L$ beyond which further reductions in loss no longer produce improvements in $A$ defines our $\varepsilon_{\text{Loss}}$ (dotted line in Fig.~\ref{fig:metrics_MLP_CNN}). Note that both the $A$ vs. $L$ relationship and the value of $\varepsilon_{Loss}$ depend strongly on the architecture and the task to be performed.

\begin{figure}
    \centering
    \includegraphics[width=0.7\linewidth]{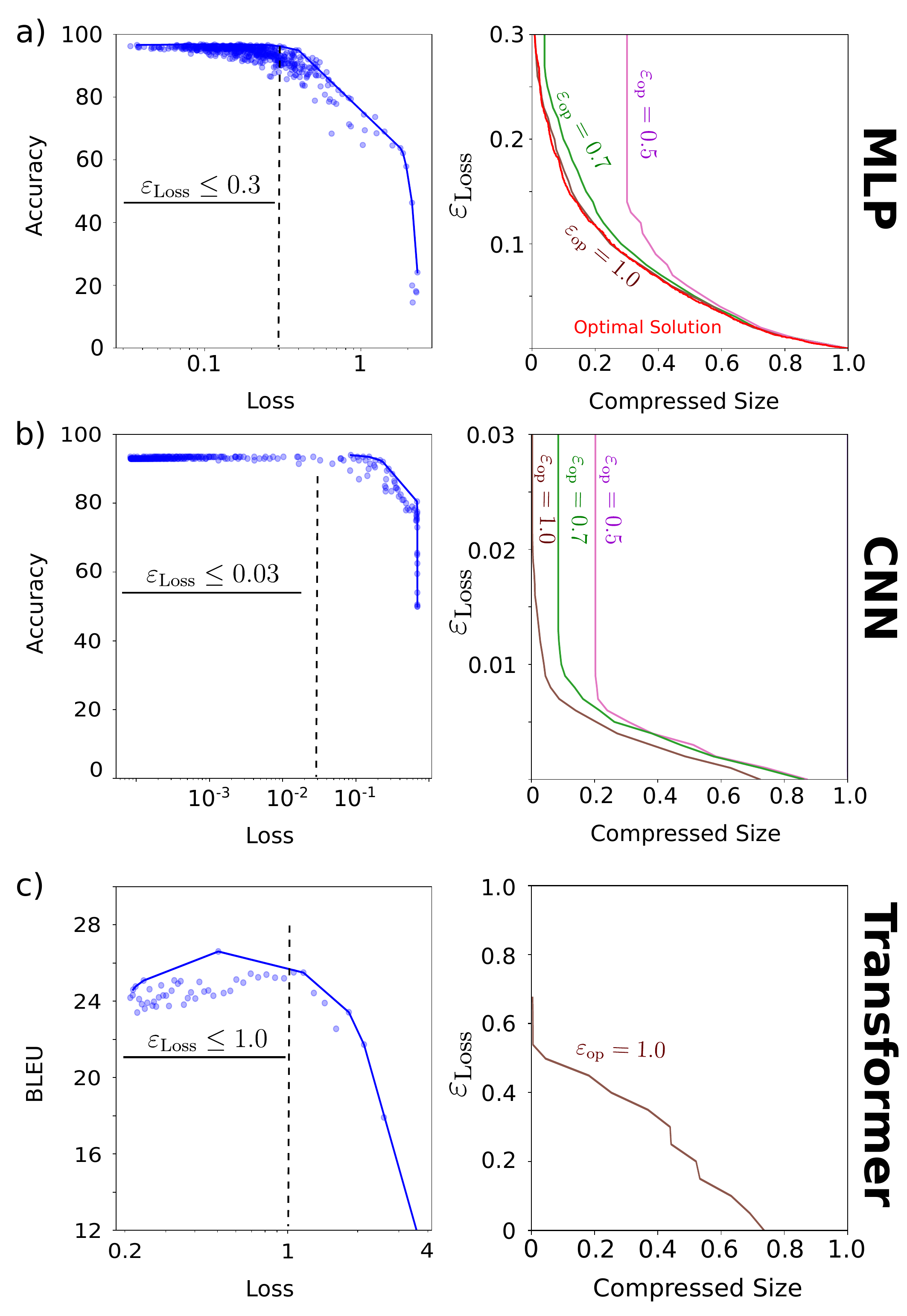}
    \caption{(Left panel) Relationship between accuracy and loss function. There exists an interval of the loss function where the network does not lose performance in terms of accuracy. This interval defines the acceptable $\varepsilon_{\rm Loss}$. (Right panel) For those acceptable $\varepsilon_{\rm Loss}$, we determine the optimal $\varepsilon_{\rm fib}$ and $\varepsilon_{\rm op}$ using various search method. Thus, each $\varepsilon_{\rm Loss}$ determines an optimal compressed size. {\bf (a)} For the MLP we used exhaustive grid search (red) in 6 dimensions $(\varepsilon^{(1)}_{\rm fib},\varepsilon^{(1)}_{\rm op}, \varepsilon^{(2)}_{\rm fib},\varepsilon^{(2)}_{\rm op}, \varepsilon^{(3)}_{\rm fib},\varepsilon^{(3)}_{\rm op})$
    and exhaustive search in Partition Space in 3 dimensions $(\varepsilon^{(1)}_{\rm fib}, \varepsilon^{(2)}_{\rm fib}, \varepsilon^{(3)}_{\rm fib})$ for fixed values of $\varepsilon_{\rm op}$ (brown, green, magenta). {\bf (b)} For the CNN we did the same exhaustive search in Partition Space in 3 dimensions. {\bf (c)} For the transformer we did Coordinate Descent in 41 dimensional space of $\vec{\varepsilon}_{\rm fib}$ and fixed values of $\varepsilon_{\rm op}$. Note that $\varepsilon_{\rm op}=1$ corresponds to enforcing quasi-fiber conditions only. BLEU (Bilingual Evaluation Understudy) is a standar metric that measures the similarity between a machine-generated text and human-written reference translations.}
    \label{fig:metrics_MLP_CNN}
\end{figure}

Note that Eq. \ref{eq:total-loss-change} is an efficient estimation of $\Delta L$ because we reduce the computational cost of computing $L$, which is usually high for large datasets. In contrast,  $\bf G$ is computed automatically during training of the original network and does not depend on $\vec{\varepsilon}$. Only $\Delta W$ dependent on $\vec{\varepsilon}$, as they determine the coloring of the nodes.

\subsection{Redundant feature pruning is a special case of fibration compression}
\label{sec:compression-permutation}

A previous work \cite{ayinde2019redundant} proposed to identify nodes that have the same input and output weights, and compress networks by removing these duplicate nodes to accelerate forward inference. Here we show that this is a special case of covering symmetry, namely, a trivial permutation of nodes in a single layer, i.e., Eq. 
(\ref{eq:fibration-definition}) and (\ref{eq:opfibration-definition}) without the sums. Their method organizes weights as follows:

\[
\mathbf{W}^{(\ell)} \mapsto 
\begin{bmatrix}
  \text{vect}(W[:,:,1,1]) & ... & \text{vect}(W[:,:,1,d_L]) \\
  \text{vect}(W[:,:,2,1]) & ... & \text{vect}(W[:,:,2,d_L]) \\
  ... & ... & ... \\
  \text{vect}(W[:,:,d_{\ell-1},1]) & ... & \text{vect}(W[:,:,d_{L-1},d_L]) \\  
\end{bmatrix}
=
\begin{bmatrix}
    w_1 & ... & w_{d_L}
\end{bmatrix}
\]
and uses the following similarity metric for agglomerative clustering
\begin{equation}
    \text{Sim}(C_a, C_b) = \frac{1}{|C_a||C_b|} \sum_{i \in C_a, j \in C_b} w_i \cdot w_j
\end{equation}

In the strictest case (sim = 1), $i,j$ are in the same cluster if $w_i = w_j$. That means, $\forall n = 1,...,d_{\ell-1}: W_{in} = W_{jn}$, which corresponds to trivial permutation symmetry.

\subsection{Orthogonal regularization is a particular case of symmetry breaking}
\label{si:orthogonal}

``Node collapse'' sometimes refers to nodes in an artificial neural network that become correlated in their activity, representing the same features of the data. Note that this phenomenon can signal the emergence of fibration symmetry due to the synchronization theorem. It was reported in \cite{papyan2020prevalence,doimo2022redundant}. Several ad hoc ways to avoid correlations exist. For example, orthogonal regularization prevents node collapse, forcing weight matrices to be orthogonal \cite{liu2022bridging}. Networks are initialized with random connections that are generally orthogonal, and this method ensures that the weight updates maintain this orthogonality.  

Assume a linear network $y = W x \in \mathbb{R}^{d_1}$ with $x \in \mathbb{R}^{d_0}$ y $W \in \mathbb{R}^{d_1 \times d_0}$. For simplicity, assume that each column of $W$ is normalized. Orthogonal regularization aims to reduce $||I-W^T W|| = ||I-1+1-W^TW|| = ||I-1+D_W||$ (see the definition of $D_W$ in \ref{material:refinement-coloring}). That is, orthogonal regularization expects $D_W$ to have zeros on the diagonal and ones everywhere else. This implies that if a clustering algorithm is applied to $D_W$, all dimensions will be in distinct clusters. In other words, the fibers in the output layer will all be trivial. Consequently, if the weights for each output node are orthogonal, no fibers exist. Then, orthogonality is a way to prevent fiber formation and avoid node synchronization

\subsection{Continual learning}
\label{material:resultsDCL}

We explore the Fibration Symmetry Breaking algorithm to resolve the problem of loss of plasticity encountered in sequential learning scenarios. Our investigation focuses on the Continual ImageNet task, comprising 5,000 binary classification tasks formed by pairing distinct classes from ImageNet. A deep neural network is trained on images from two classes per task and evaluated on a corresponding test set. ImageNet, with 1,000 classes and approximately 700 images per class, provides the dataset. Each class is divided into 600 training images and 100 test images. Tasks are presented sequentially, with a new task sampling a new pair of classes. Training occurs over 250 epochs per task. Performance, assessed by the accuracy of the test set, highlights a loss of plasticity if the accuracy decreases.

We utilize a CNN with three convolutional layers and three fully connected layers, with the final layer consisting of two nodes for the current task's classes. The output layer is reinitialized for each new task, following common practice. We evaluate standard gradient descent, deep continual learning, and symmetry-breaking gradient descent, all minimizing cross-entropy loss. Each method undergoes 10 independent runs with the same sequence of class pairs, and the network weights are initialized once before the first task.

The results, illustrated in Fig.~\ref{fig:results-formation-and-pruning}g, are representative of a feed-forward convolutional network using unmodified backpropagation. Initially, networks achieved up to 88\% accuracy on early tasks but lost plasticity by the 2,000th task. This pattern was consistent across architectures, parameters, and optimizers, indicating standard deep-learning methods struggle with continual-learning tasks. 

The emergence of covers during learning is exemplified in Fig.~\ref{fig:distrib_fibers} for the sequential ImageNet task. At the start (Epoch 0) all 512 nodes in the layer constitute trivial fibers (of size 1). After some training epochs, large covers emerge. Continual backpropagation and even more so symmetry breaking keep the cover size in check and therefore have a larger number of them, across the 250 epochs of the first task (first row) and even more clearly after 900 tasks (second row).    

To decide when and how many nodes to reinitialize, we have to select a particular schedule. For easier comparison, we use the same re-initialization schedule as in \cite{dohare2024loss}. Namely, nodes are re-initialized at a fixed replacement rate (which is the number of node per gradient step), taken from covers only if they reach a certain size maturity threshold. We trained the network with a momentum term. The results in Fig.~\ref{fig:results-formation-and-pruning}g are for a learning rate of 0.01, a replacement rate of 3e-4, a maturity threshold of 100 and a decay rate of 0.99. Minibatch sizes of 100, a momentum of 0.9.

\begin{figure}[ht]
\centering
\includegraphics[scale=0.24]{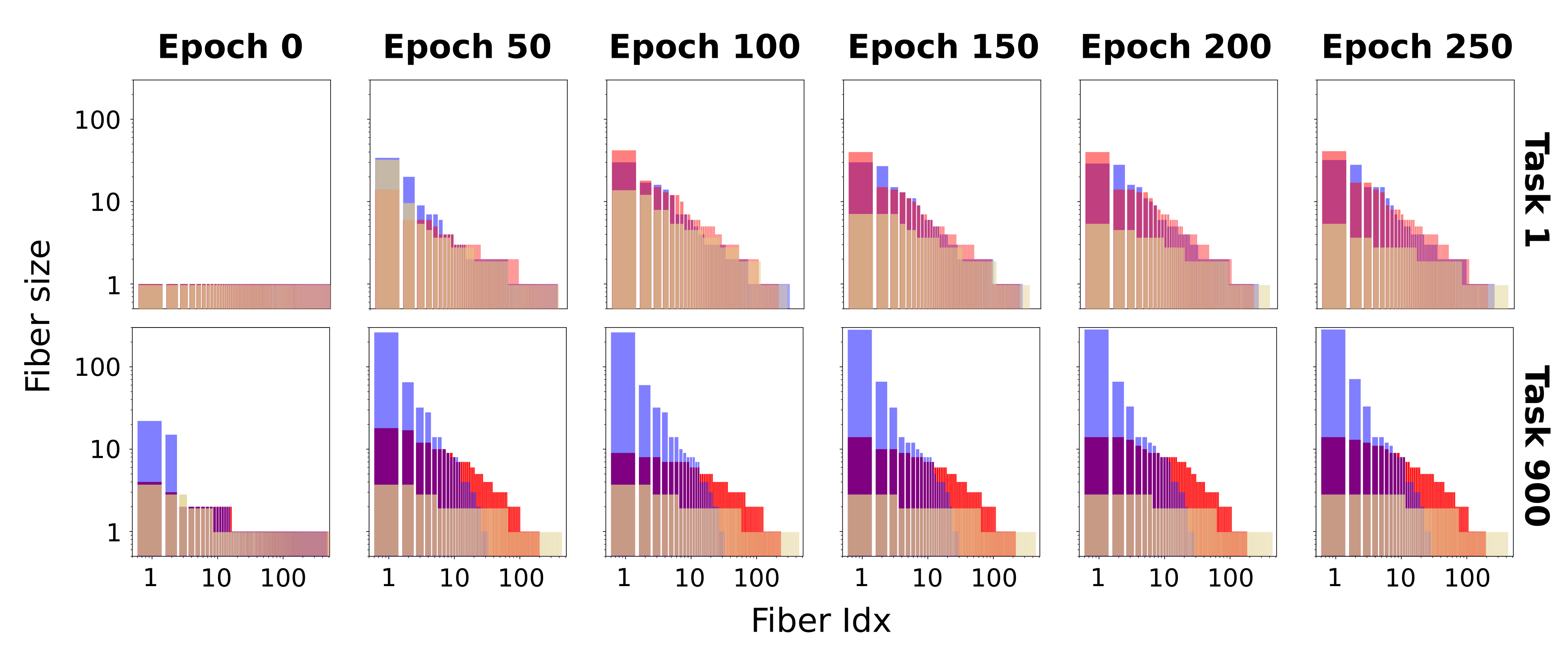}
\caption{\textbf{Distribution of covers, fibers and opfibers during continual learning on the Sequential ImageNet Task.} Evolution of the size and number of covers across leaning epochs (columns) and task number (rows) for different training methods: (blue) SGD (red) continual backpropagation (yellow) FSB. This is the distribution of sizes for the second dense layer.}
\label{fig:distrib_fibers}
\end{figure}

\subsection{Deep Continual Learning is a particular case of Fibration Symmetry Breaking}
\label{material:DCL}

In Deep Continual Learning (DCL) algorithm proposed in \cite{dohare2024loss}, the contribution utility  is
\begin{align*}
    u^{(\ell)}_i(t+1) &= \eta. u^{(\ell)}_i(t) + (1-\eta) F^{(\ell)}_i(t), \\
    F^{(\ell)}_i(t) &= h^{(\ell)}_i(t).\sum_{k=1}^{d_{\ell+1}} W^{(\ell+1)}_{ki} (t)
\end{align*}
with $u^{(\ell)}_i(0) = 0$. 

If there are two nodes $i \underset{cov,\ell,t}{\sim} j$, then $F^{(\ell)}_i(t) = F^{(\ell)}_j(t)$ because $h^{(\ell)}_i = h^{(\ell)}_j$ for fibration symmetry and $\sum_{k=1}^{d_{\ell+1}} W^{(\ell+1)}_{ki} (t) = \sum_{k=1}^{d_{\ell+1}} W^{(\ell+1)}_{kj} (t)$ for opfibration symmetry. 

If there are two nodes $i \underset{cov,\ell,t}{\sim} j$ with identical utilities $u^{(\ell)}_i(t) = u^{(\ell)}_j(t)$, then $u^{(\ell)}_i(t+1) = u^{(\ell)}_j(t+1)$ (for definition of $u$) and $i \underset{cov,\ell,t+1}{\sim} j$. Therefore, we obtain $u_i(s) = u_j(s) \quad \forall t \leq s$. That means that \textit{covering symmetry implies utility synchronization}.

Deep Continual Learning resets nodes belonging to the same cover when their utility reaches zero. For these nodes, input weights are reinitialized randomly, and output weights are fixed at zero. This
modifies the node's activity $h$ while avoiding immediate disruption of the learned network behavior. Altering both input and output
weights necessarily modifies the input and output trees of the nodes, thus breaking their associated covers. In particular, zero-utility nodes typically exhibit null activity ($h=0$) and belong to the same fiber. Thus, DCL inherently targets the breaking of inactive covers. The proposed fibration symmetry-breaking protocol, on the other hand, targets all covers, not only the inactive ones. This explains why the fibration breaking approach overcomes DCL in continual learning in Fig. \ref{fig:results-formation-and-pruning}g.

\subsection{Fibers are a refinement of clustering partitions}
\label{material:fibers-clusters}

Consider a  classification problem where each input $x$ has to be classified by output $y$ into one of the $N$ classes. For each class$c$, we define $P^{(\ell)}_c$ as the partition obtained by the correlation matrix of activity with the examples drawn from class $c$. Each set $A \in P^{(\ell)}_c$ is a cluster of nodes in layer $\ell$ where the activity is synchronized for all inputs in class $c$ (i.e. $h_i(x) = h_j(x) \quad \forall x \in c$). We refer to these as class-conditional synchronization clusters. Examples are shown in Fig. \ref{fig:clusters_class}.

\begin{figure}[ht]
\centering
\includegraphics[width=0.9\linewidth]{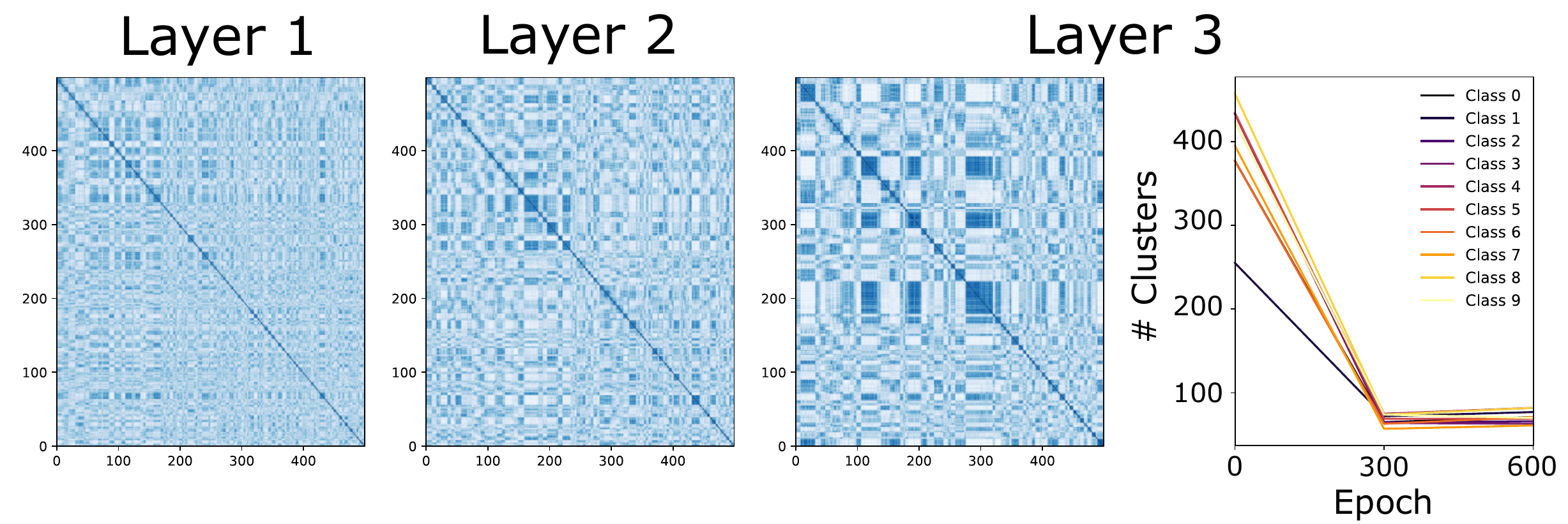}
\caption{Class-conditional synchrony clusters for class 9 in multiples layers of the MLP trained on MNIST. On the right, we show the number of clusters for layer 3 as function of learning epochs for all 10 classes. }
\label{fig:clusters_class}
\end{figure}

Compute the intersection of clusters from different classes
\begin{equation*}
  I = A_1 \cap A_2 \cap ... \cap A_{N}  
\end{equation*}

Two nodes in $I$ are synchronized for all inputs in the data set; however, they may not be synchronized for an input that is not part of the data set. That is, a fiber is always contained within a set like $I$. In other words, the partition formed by the fibers is a refinement of the partition formed by sets of the form $I$ (i.e. intersection of clusters across classes in Fig. \ref{fig:sync}b).

To test this theory numerically, we threshold the correlation coefficients for the synchrony clusters as well as the fibration condition - Eq. (\ref{eq:fibration-definition}), and then measure a refinement score between these two partitions of nodes (Fig. \ref{fig:scores}b). The \textbf{refinement score}, counts the number of fibers that have a unique cluster label normalized by the total number of fibers. For small thresholds, the theoretical result is validated (refinement score=1) while for larger cluster thresholds the relationship breaks down (refinement score=0), as expected. This observation is more generally true for any subset of all possible inputs, including inputs generated at random.

The synchronization theorem states that fibers imply synchronization for any arbitrary input, but the reverse is not necessarily true. We test the reverse of this statement empirically by comparing the fiber partition with synchronization clusters for random inputs. We measure how well the two partitions match using as \textbf{matching score}. For this we use the ``clustering accuracy", which is the maximized sum of the diagonal of the confusion matrix divided by the total number of items. The maximum assignment of partition labels is found with the linear sum assignment algorithm.  For small thresholds, the fiber partition and synchronization clusters are a perfect match (matching score=1, Fig.~\ref{fig:scores}a), showing that fibers can also be identified simply by correlating nodes under random inputs.  

\begin{figure}[ht]
\centering
\includegraphics[width=0.7\linewidth]{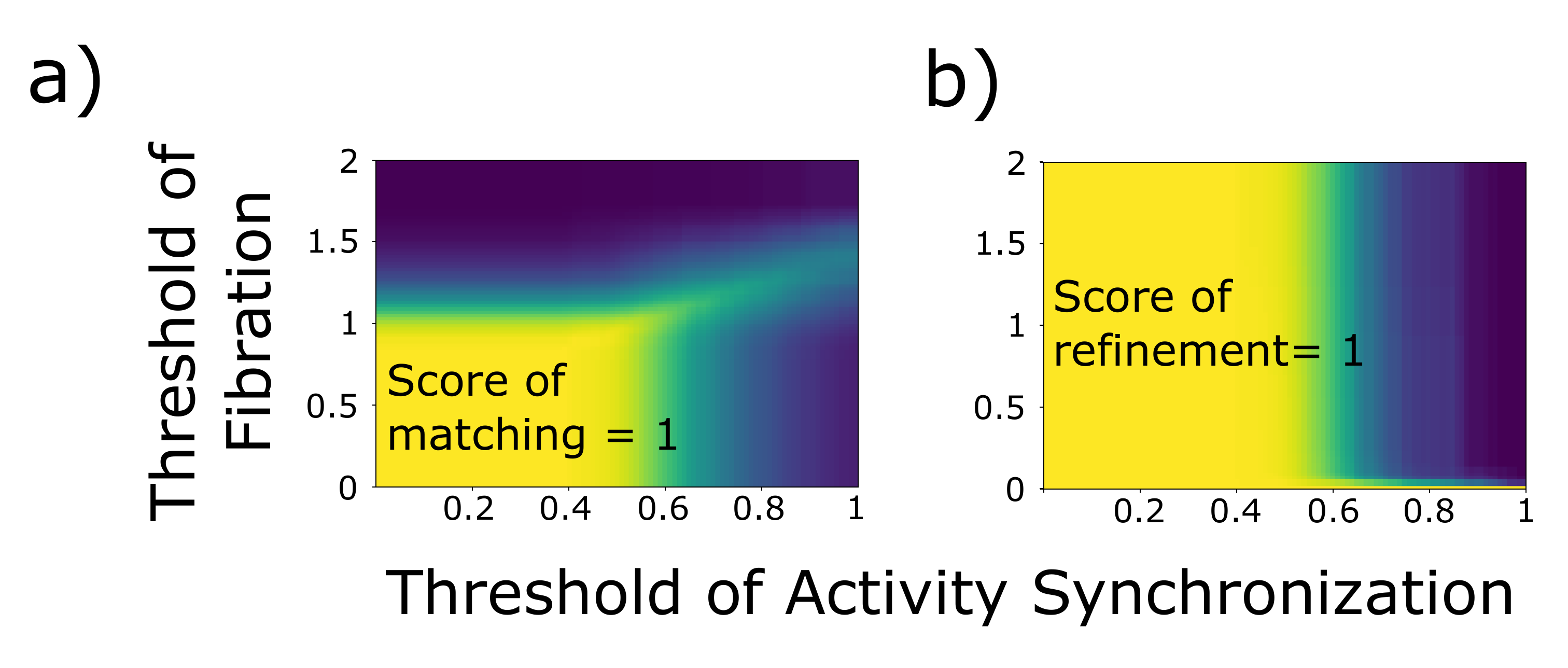}
\caption{Matching score and refinement score compare the these clusters with fibers for different thresholds of identifying the clusters (horizontal) and fibers (vertical)}
\label{fig:scores}
\end{figure}

\subsection{Hierarchical organization of the loss landscape under SGD dynamics}
\label{sec:hierarchical}

Cover formation is a hierarchical coarse-graining under the SGD
dynamic in the parameter and node space. This is exemplified in
Fig.~\ref{fig:sync}b-c, where we show the state of the network in
three different training stages ($T_1, T_2, T_3$), their transitions
$T_{1\to2}$ and $T_{2\to3}$, and the evolution of three nodes $(i, j,
k)$ as covers merge.

Covers are a refinement of fibers. Fibers are a refinement of
synchronization clusters, and these in turn are a refinement of
class-conditional synchrony clusters. This is exemplified in
Fig. \ref{fig:sync}b, along with the development of covers and
class-conditional clusters during learning. The red box groups the
class-conditional clusters for the red class (and similarly
for the blue class) at time $T_1$. Covers in green are a refinement of both red and blue clusters. As learning progresses, the covers merge at transitions $T_{1\to2}$ and $T_{2\to3}$ according to the
Cover Coarse-Graining Theorem.

Figure \ref{fig:sync}c is a symbolic representation of a parameter space (weights). The potential learning trajectories (black curves) merge the covers containing nodes $i, j$ and $k$. During the learning phase $T_1$, the three nodes $i, j,$ and $k$ belong to distinct covers. At a later point in the training, $i$ and $j$ merge into the same cover. According to the Cover Coarse-Graining Theorem, once this occurs, they cannot separate again. In other words, the first time $i$ and $j$ share a cover defines an irreversible transition (green dashed line $T_{1\to2}$). A similar transition occurs when the cover containing node $k$ merges into the cover containing $i$ and $j$, defining a second transition line $T_{2\to3}$. The colored zones in this space are regions where the network has the same cover base $B$ (with identical nodes and weights). 

The loss landscape $\mathcal{L}(w)$, with $w \in \mathbb{R}^d$, is
characterized by a hierarchy that can be formally represented as a
tree (Fig. \ref{fig:sync}d). All network configurations in this
$d$-dimensional parameter space that have the same minimal base $B''$
will have the same loss value. The root of the tree is this base in
$\mathbb{R}^{d''}$; while the leaf nodes are all the possible lifting
of the base in $\mathbb{R}^d$. The tree contains multiple leaves
because $\varphi_{\rm fib}$ can reduce different graphs $G$ into the
same minimal base $B''$ (This is evident in
Eq. (\ref{eq:fib-weight-compression}), as a linear subset of all
possible $W^{(\ell)}$ can generate the same $\hat{W}^{(\ell)}$). The
inverse lifting operation can transform a representation from
$\mathbb{R}^{d''}$ (minimal base $B''$) to $\mathbb{R}^d$ (original
network $G$) through an intermediate lifting to $\mathbb{R}^{d'}$
where $d'' < d' < d$. These intermediate steps create the hierarchy of
the trees (see Methods \ref{material:lifting} -
Fig. \ref{fig:lifting}).

The specific shape of the cost function in this hierarchy can appear
as disjoint local minima in the full parameter space of $G$ (orange
curves, point 3.1) corresponding to a single point the parameter
space of the base $B''$. In other cases, the minimum can be flat
valleys in the loss landscape (purple surface, point
2.1). Elsewhere, the structure of the loss function is replicated,
creating more complex patterns (green curves, point 1).

The learning trajectory navigates downhill in this loss landscape
(solid black line, in Figs. \ref{fig:sync}c-d) following the hierarchical
coarsening of the covers under the SGD learning dynamics, where the system progressively explores increasingly coarse-grained symmetry quotients during
optimization. This structure has analogies in the ultrametric organization of associative memory in neural networks \cite{parga1986theultrametric} and the energy landscape in spin glasses \cite{mezard1987spin}.

\begin{figure}[ht!]
\includegraphics[scale=.08]{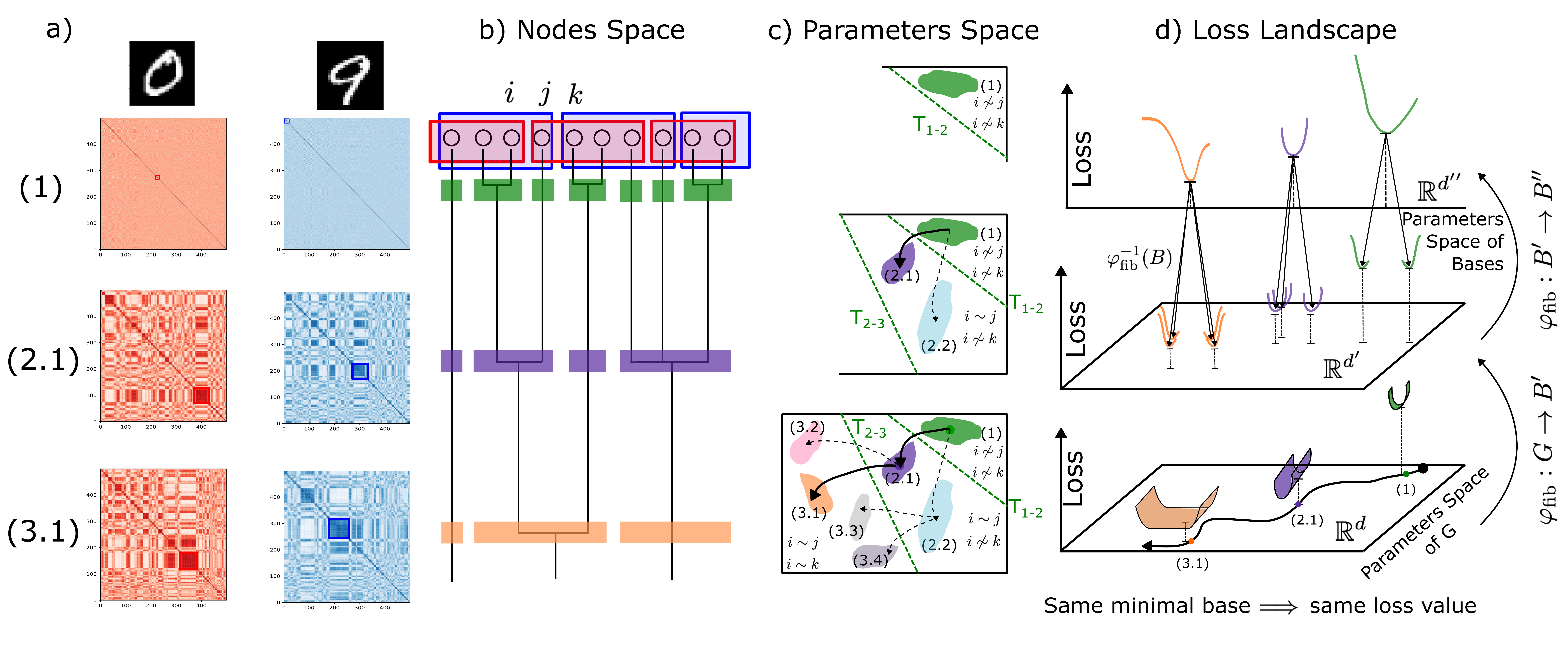}
\caption{ \textbf{Synchronization, hierarchical clustering and loss
    landscape.}  
  {\bf (a)} Class-conditional correlation matrix of
  activity with network inputs drawn from individual classes. Here
  they are shown for the last layer of an MLP trained on MNIST in early epochs (1), in middle of the training (2.1) and after training (3.1). Clusters in this
  matrix are referred to as synchrony clusters.  
  {\bf (b)} Temporal
  evolution of covers in one layer. In one of the early epochs, (1), we identify synchrony
  clusters (red and blue) and covers (green). Covers are located within cluster
  intersections. In later epochs, (2.1) and (3.1), covers merged to
  form coarser partitions (purple and orange). Nodes $i$, $j$ and $k$
  start in different covers (green), but merge into the same cover (purple and orange).  
  {\bf (c)}
  A network is a point in parameter space, e.g. at time point (1). Multiple networks can have the same basis $B$ (green region). Learning trajectories (black curves) merge nodes $i$ and $j$ into the same cover at an
  irreversible transition boundary $T_{1-2}$. Multiple $B$ may have
  $i$ and $j$ in the same cover (e.g., network (2.1) and
  (2.2)). SGD selects one of these possibilities. Once $k$ merges into
  the same cover, another irreversible transition occurs $T_{2-3}$.
  {\bf (d)} Hierarchical loss landscape (see text for details).}
\label{fig:sync}
\end{figure}

\subsection{Low-rank weight compression}
\label{sec:compression-SVD}

An existing heuristic for network compression identifies input weights that are (nearly) collinear using a singular value decomposition (SVD) of the weight matrix. The weight matrix in any layer of a network can be factorized, removing null or low-variance dimensions, without substantially altering the function of the network. The factorization then has a smaller parameter count compared to the original network. Several methods were proposed simultaneously following this idea, e.g. \cite{jaderberg2014speeding,denton2014exploiting,novikov2015tensorizing}. In modern LLMs, the removal of low-variance dimensions of weight matrices can even improve performance without further tuning \cite{sharma2023truth}, and can achieve compression factors of $\approx$3x \cite{tomut2024compactifai}. 

In principle, compression with SVD followed by fibration compression can be easily combined. SVD factors a matrix as $W \approx W^{(2)} W^{(1)}$ where $W^{(2)}$ has $k$ orthonormal columns and $W^{(1)}$ has $k$ orthogonal rows. Essentially, this is equivalent to adding an intermediate linear layer with $k$ nodes. The approximation is exact if $k$ is the rank of $W$. With $d_{\rm out} d_{\rm in}$ parameters for $W$, the factorization has $k(d_{\rm out}+d_{\rm im})$ parameters, which is more compact if one selects $k<d_{\rm out} d_{\rm in}/(d_{\rm out}+d_{\rm in})$. To make this effective, the dimension $k$ of the subspace must be selected small, at the risk of incurring an error, which is typically corrected by the subsequent fine-tuning of the matrices $W^{(1)}$ and $W^{(2)}$. 
For instance, in a square matrix $d_{\rm in}=d_{\rm out}=1,000$ the rank $k$ needs to be smaller than $500$ for a compression to be effective. Then, a parameter compression factor of 4x would require a rank compression of 8x. 
Applying a fibration compression to this simply reduces the size of the new individual factors matrices: 
\begin{equation*}
    \hat{W}_{c'c} 
    = \frac{1}{|c'|} \sum_{i \in c'} \sum_{j \in c} \sum_k W_{ik}^{(2)} W_{kj}^{(1)}
    = \sum_k \hat{W}_{ik}^{(2)} \hat{W}_{kj}^{(1)}   .
\end{equation*}
For this representation to still be smaller, $k$ has to be now smaller than $k<\hat{d}_{\rm out} \hat{d}_{\rm in}/(\hat{d}_{\rm out}+\hat{d}_{\rm in})$, i.e. with the smaller node dimensions after fibration compression. So SVD first followed by fibration compression can be readily combined, but the compression efficiency of SVD is potentially lost. An SVD compression after fibration compression may potentially find additional co-linear nodes that are not identified by fibration compression. 

From our current perspective, we can prove that if the input of any layer lies in a linear subspace, then the input weight matrix of that layer will be attracted by SGD into that subspace. Specifically, if $W$ is in the linear subspace of the input $X$, then so is the gradient $\partial L/\partial W = \Delta X^T$ (both $X$ and $\Delta$ are here matrices of dimensions $\times$ exemplars). Therefore, the linear subspace defined by the data becomes an invariant set for the weight matrix, and will be an stable attractor for SGD, following \cite{chen2023stochastic}. For high-dimensional representation layers, it is quite possible that the input is rank-deficient, which would explain why low-rank weight matrices are found empirically and are amenable to compression.  

SVD of a weight matrix $W$ reveals a symmetry of algebraic nature: it decomposes the transformation into pure rotations $U$, $V$, and scalings $\Sigma$. However, this type of symmetry is not part of the global-local hierarchy of graph symmetries shown in Fig. \ref{fig:hierarchy}. Graph symmetries define a partition of nodes, where nodes in the same class are structurally indistinguishable under one criterion (e.g., isomorphic input trees in fibration symmetries). Although graph symmetries identify structural redundancies, SVD uncovers the underlying geometric simplicity of the mapping $W = U \Sigma V^T$.

\clearpage

\section*{Acknowledgements}
Partial support for this work was provided by the National Institutes of Health through grants R01CA247910 (OMV, LCP, HAM) and R01EB028157 (HAM), as well as the Army Research Office with grant WF911-NF-24-1-0031 (OMV, LCP). We thank Matteo Serafino for his help and discussions.

\section*{Declarations}

\paragraph*{Funding:}
Partial support for this work was provided by the National Institutes of Health through grants R01CA247910 (OMV, LCP, HAM) and R01EB028157 (HAM), as well as the Army Research Office with grant WF911-NF-24-1-0031 (OMV, LCP).

\paragraph*{Conflict of interest/Competing interests:}
There are no competing interests to declare.

\paragraph*{Ethics approval and consent to participate:}
N/A

\paragraph*{Consent for publication:}
N/A

\paragraph*{Data, materials, and code availability:}
The code is available at \href{github.com/MakseLab/fibrations_in_dnns}{Github/MakseLab/fibrations$\_$in$\_$dnns}

\paragraph*{Author contributions:}
HAM conceived the project. OMV, HAM, and LCP designed the experiments, contributed to the mathematical proofs, and wrote the paper. All authors contributed to the analysis of data. OMV and AH developed the analysis code.

\end{document}